%% file: main.tex
\documentclass[sigconf]{acmart}

\usepackage{booktabs} % For formal tables
\usepackage[ruled,vlined,linesnumbered]{algorithm2e}
\usepackage{multirow}
\usepackage{balance}
\usepackage{subcaption}
\usepackage{dsfont}
\usepackage{bbm}
\usepackage{footnote}
\usepackage{paralist}
\usepackage{enumitem}
\usepackage{xcolor, colortbl}
\usepackage{algpseudocode}
\usepackage{amsmath}

\newcommand{\jian}[1]{{\small\color{brown}{\bf Jian: #1}}}
\newcommand{\yinglong}[1]{{\small\color{teal}{\bf Yinglong: #1}}}

\newcommand{\hide}[1]{}

\newcommand{\rawlsgcn}{\textsc{RawlsGCN}}

\DeclareMathOperator*{\argmin}{argmin}

\definecolor{Gray}{gray}{0.75}
\newcolumntype{a}{>{\columncolor{Gray}}c}

\newtheorem{prob}{Problem}
\newtheorem{lm}{Lemma}
\newtheorem{thm}{Theorem}
\newtheorem{defn}{Definition}

\AtBeginDocument{%
  \providecommand\BibTeX{{%
    \normalfont B\kern-0.5em{\scshape i\kern-0.25em b}\kern-0.8em\TeX}}}

\copyrightyear{2022} 
\acmYear{2022} 
\setcopyright{acmcopyright}
\acmConference[WWW '22]{Proceedings of the ACM Web Conference 2022}{April 25--29, 2022}{Virtual Event, Lyon, France}
\acmBooktitle{Proceedings of the ACM Web Conference 2022 (WWW '22), April 25--29, 2022, Virtual Event, Lyon, France}
\acmPrice{15.00}
\acmDOI{10.1145/3485447.3512169}
\acmISBN{978-1-4503-9096-5/22/04}

\begin{document}
% \fancyhead{}
% \mkclean{}
\title{\rawlsgcn: Towards Rawlsian Difference Principle on Graph Convolutional Network}

%% Author
\author{Jian Kang$^1$,~~~Yan Zhu$^2$, ~~~Yinglong Xia$^2$, ~~~Jiebo Luo$^{2,3}$, and Hanghang Tong$^1$}
\affiliation{
	\institution{
		$^1$University of Illinois at Urbana-Champaign, \{jiank2, htong\}@illinois.edu; \\
		$^2$Facebook AI, \{yzhu, yxia\}@fb.com; \\
		$^3$University of Rochester, \{jluo\}@cs.rochester.edu;
	}
	\country{}
}
\renewcommand{\shortauthors}{Kang et al.}

\input{00abs}




\maketitle

% \vspace{-3mm}
\input{01intro}
% \vspace{-3mm}
\input{02prelim}
% \vspace{-3mm}
\input{03method}
% \vspace{-4mm}
\input{04experiment}
\vspace{-3mm}
\input{05related}
\vspace{-3mm}
\input{06conclusion}
\vspace{-3mm}
\section*{Acknowledgement}
\input{08acknowledgement.tex}

\bibliographystyle{ACM-Reference-Format}
\bibliography{ref}

\eject
\balance
\input{07appendix}

\end{document}

%% file: 00abs.tex
\begin{abstract}
Graph Convolutional Network (GCN) plays pivotal roles in many real-world applications. Despite the  successes of GCN deployment, GCN often exhibits performance disparity with respect to node degrees, resulting in worse predictive accuracy for low-degree nodes. We formulate the problem of mitigating the degree-related performance disparity in GCN from the perspective of the Rawlsian difference principle, which is originated from the theory of distributive justice. Mathematically, we aim to balance the utility between low-degree nodes and high-degree nodes while minimizing the task-specific loss. Specifically, we reveal the root cause of this degree-related unfairness by analyzing the gradients of weight matrices in GCN. Guided by the gradients of weight matrices, we further propose a pre-processing method \rawlsgcn-Graph and an in-processing method \rawlsgcn-Grad that achieves fair predictive accuracy in low-degree nodes without modification on the GCN architecture or introduction of additional parameters. Extensive experiments on real-world graphs demonstrate the effectiveness of our proposed \rawlsgcn\ methods in significantly reducing degree-related bias while retaining comparable overall performance.
\end{abstract}

%% file: 01intro.tex
\section{Introduction}\label{sec:intro}
%{\yinglong{Introduction is a bit wordy. You may want to make it succinct, such as the paragraph of “In this paper…” and that of the main contributions.} }\hh{i shrank it a bit}
Graph structured data naturally appears in many real-world scenarios, ranging from social network analysis~\cite{peng2016social}, drug discovery~\cite{chen2019graph}, financial fraud detection~\cite{zhang2017hidden} to traffic prediction~\cite{derrow2021eta}, recommendation~\cite{wang2019neural} and many more. The success of deep learning on grid-like data has inspired many graph neural networks in recent years. Among them, Graph Convolutional Network (GCN)~\cite{kipf2017semi} is one of the most fundamental and widely used ones, often achieving superior performance in a variety of tasks and applications. 

Despite their strong expressive power in node/graph representation learning, recent studies show that GCN tends to under-represent nodes with low degrees~\cite{tang2020investigating}, which could result in high loss values and low predictive accuracy in many tasks and applications. As shown in Figure~\ref{fig:intro_example}, it is clear that low-degree nodes suffer from higher average loss and lower average accuracy in semi-supervised node classification. Such a performance disparity w.r.t. degrees is even more alarming, given that node degrees of real-world graphs often follow a long-tailed power-law distribution which means that a large fraction of nodes have low node degrees. In other words, the overall performance of GCN might be primarily beneficial to a few high-degree nodes (e.g., celebrities on a social media platform) but biased against a large number of low-degree nodes (e.g., grassroot users on the same social media platform). %Consider a user-user graph of a commercial social media platform where low-degree nodes refer to the users who do not receive many engagement signals (e.g., like, comment, share) from other users. If a GCN is deployed to recommend user posts to other users, low-degree users suffering from low predictive accuracy will be discouraged by the system because the system aims to make predictions as accurate as possible in order to maximize its profit. Thereby, those low-degree users might quit using the platform due to lack of engagements and poor user experience. And the platform would suffer from user growth in the long run because it only reveals good user experience for the high-degree users (e.g., celebrities).
% \hh{from here to the next paragraph, it is not very smooth and reads a bit jumpy. i think a more natural way might be (1) what is the state of the art to debias the degree-related in GNN (e.g., 21, 27 and 31). (2) what are the main limitation of these works (make sure to align the limitations with our main contributions) (3) then, we talk about our works, e.g. the main ideas and contributions}

\begin{figure}
    \centering
    \includegraphics[width=.43\textwidth]{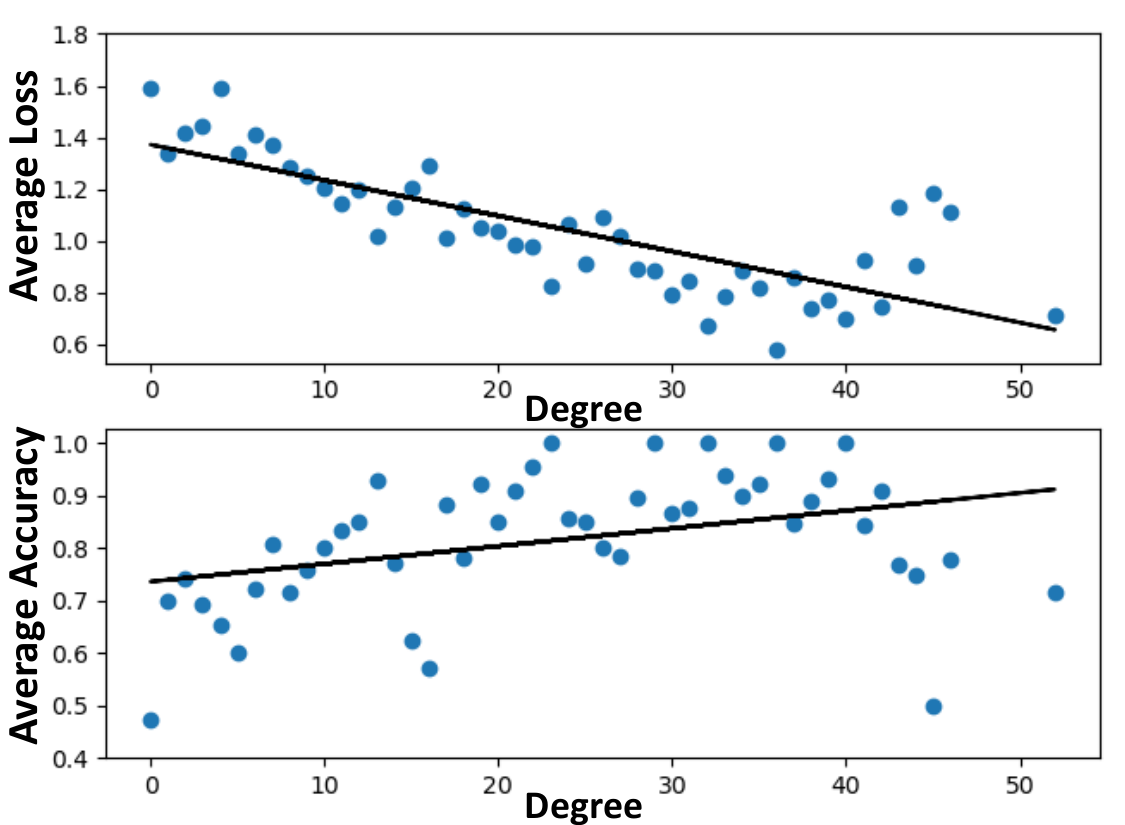}
    \vspace{-2mm}
    \caption{An illustrative example of the underrepresentation of low-degree nodes in semi-supervised node classification. A 2-layer GCN is trained on the Amazon-Photo dataset. Blue dots refers to the average loss and average accuracy of a specific degree group (i.e., the set of nodes with the same degree) in the top and bottom figures, respectively. Black lines are the regression lines of the blue dots in each figure. For visualization clarity, we only consider the degree groups that contain more than five nodes.}
    \label{fig:intro_example}
    \vspace{-8mm}
\end{figure}

To date, only a few efforts have been made to improve the performance for low-degree nodes. For example, DEMO-Net~\cite{wu2019net} randomly initializes degree-specific weight parameters 
%\hh{is it true that demo-net randomly generate the weight matrices? check with Jun} \jian{yes, i changed the word `generates' to `initializes' since those weights are learnable as well} 
in order to preserve the neighborhood structure for low-degree nodes. SL-DSGCN~\cite{tang2020investigating} introduces a 
% RNN-based 
degree-specific weight generator 
based on recurrent neural network (RNN) and a semi-supervised learning module to provide pseudo labels for low-degree nodes. Recently, Tail-GNN~\cite{liu2021tail} proposes a novel neighborhood translation mechanism to infer the missing local context information of low-degree nodes. However, the fundamental cause of degree-related unfairness in GCN largely remains unknown, which in turn prevents us from mitigating such unfairness from its root. Furthermore, existing works introduce either additional degree-specific weight parameters~\cite{wu2019net, tang2020investigating} or additional operation on node representations (e.g., forged node generator, neighborhood translation)~\cite{liu2021tail}, which significantly increase the computational cost in both time and space and thus hinder the scalability of these models.%\hh{Jian, ij shortened it a bit. pls check to see if it makes sense.}\jian{checked. this version is much more succinct.}%almost all existing works suffer from the following limitations.
\hide{
\begin{compactitem}
    \item [\textbf{L1.}] There is limited understanding on the root cause of degree-related unfairness in GCN. Without more understanding on the root cause, existing works might not effectively resolve the degree-related unfairness. As a result, the unfairness is hardly mitigated from the origin. 
    \item [\textbf{L2.}] In order to improve the performance of low-degree nodes, existing works introduce either additional degree-specific weight parameters~\cite{wu2019net, tang2020investigating} or additional operation on node representations (e.g., forged node generator, neighborhood translation)~\cite{liu2021tail}, which are essentially new graph neural network architectures with additional number of parameters to be learned. Moreover, the newly introduced weight parameters or operations significantly increases the computational cost in both time and space, which hinders the scalability of these models. \textit{How to mitigate the degree-related unfairness on vanilla GCN without changing its architecture or introducing additional parameters} is a question that remains nascent. \yinglong{Without going through the rest part of the paper, it can be a bit confusing why do we really care if the GCN is a vanilla version or not for mitigating the unfairness? The narrative can be improved.} \jian{I added one more sentence to say that these additional weight/operations will increase the computational cost.}
\end{compactitem}
}

In order to tackle these limitations,
we introduce the Rawlsian difference principle~\cite{rawls1971theory} to mitigate degree-related unfairness in GCN. As one of the earliest definitions of fairness from the theory of distributive justice% as fairness
%\hh{this term seems broken (two fairness). pls check}\jian{i changed it to distributive justice.}
~\cite{rawls1971theory}, the Rawlsian difference principle aims to maximize the welfare of the least fortunate group and achieves stability when the worst-off group seeks to preserve its status quo. In the context of GCN, it requires a GCN model to have balanced performance among the groups of nodes with the same degree when the Rawlsian difference principle is in its stability. Given its essential role in training the GCN through backpropagation, we focus our analysis on the gradients of weight matrices of GCN. In particular, we establish the mathematical equivalence between the gradient of the weight matrix in a graph convolution layer and a weighted summation over the influence matrix of each node, weighted by its corresponding node degree in the input of GCN. This analysis not only reveals the root cause of degree-related unfairness in GCN, but also naturally leads to %. We define the influence of each node as the synergy between its own node features and the gradients of node features in its local neighborhood. Then our analysis reveals that, in a graph convolution layer, the gradient of the weight matrix is equivalent to the summation of the node influence, weighted by the corresponding node degree in the input adjacency matrix. For L2, guided by the gradient of weight parameters, we propose 
two new methods to mitigate the performance disparity with respect to node degrees, including (1) a pre-processing method named \rawlsgcn-Graph that precomputes a doubly stochastic normalization of the input adjacency matrix to train the GCN; and (2) an in-processing method named \rawlsgcn-Grad that normalizes the gradient so that each node will have an equal importance in computing the gradients of weight matrices.

The main contributions of this paper are summarized as follows.
\begin{compactitem}
    \item \textbf{Problem.} To our best knowledge, we are the first to introduce the Rawlsian difference principle to GCN so as to mitigate the degree-related unfairness. %To our best knowledge, we are the first to incorporate the Rawlsian difference principle in graph neural networks.
    
    \item \textbf{Analysis.} We reveal the mathematical root cause of the degree-related unfairness in GCN. %by analyzing how the gradients of the weight matrices are computed. %We prove that the gradient of the weight matrix in a graph convolution layer is equivalent to a weighted summation over the influence matrix of each node, weighted by its corresponding node degree.
    
    \item \textbf{Algorithms.} We propose two easy-to-implement methods to mitigate the degree bias, with no need to change the existing GCN architecture or introduce additional parameters.
    
    \item \textbf{Evaluations.} We perform extensive empirical evaluations on real-world graphs, which demonstrate that our proposed methods (1) achieve comparable accuracy with the vanilla GCN, (2) significantly decrease the degree bias, and (3) take almost the same training time as the vanilla GCN. Surprisingly, by mitigating the bias of low-degree nodes, our methods can sometimes improve the overall classification accuracy by a significant margin.
\end{compactitem}

The rest of this paper is organized as follows. Section~\ref{sec:prelim} formally defines the problem of enforcing the Rawlsian difference principle on GCN. Our mathematical analysis and the proposed methods are introduced in Section~\ref{sec:method}. Section~\ref{sec:experiment} presents the experimental settings and evaluations. We review the related work in Section~\ref{sec:related}. Finally, Section~\ref{sec:conclusion} concludes the paper.

%% file: 02prelim.tex
\section{Problem Definition}\label{sec:prelim}
%\vspace{-1mm}
In this section, we first introduce the main symbols used throughout the paper in Table~\ref{tab:symbol}. Then we present a brief review on the Graph Convonlutional Network (GCN) and the Rawlsian difference principle. Finally, we formally define the problem of enforcing the Rawlsian difference principle on GCN.

\begin{table}[ht!]
	\centering
	\vspace{-3mm}
	\caption{Table of symbols.}
	\vspace{-4mm}
	\label{tab:symbol}
	\begin{tabular}{c|c}
		\hline
		\textbf{Symbols} & \textbf{Definitions and Descriptions} \\
		\hline
		$\mathbf{A}$ & a matrix \\
		$\mathbf{A}^T$ & transpose of matrix ${\mathbf A}$ \\
		$\mathbf{u}$ & a vector \\
		\hline
		$\mathcal{G}$ & a graph \\
		$\mathcal{V}$ & a set of nodes \\
		$J(\cdot)$ & objective function \\
		$\sigma(\cdot)$ & activation function \\
		$\mathbf{X}$ & node feature matrix \\
		$\mathbf{H}^{(l)}$ & node representations at $l$-th layer \\
		$\mathbf{W}^{(l)}$ & weight matrix at $l$-th layer \\
		$L$ & number of graph convolution layers \\
		$d_l$ & hidden dimension of $l$-th layer \\
		$\textit{deg}(u)$ & degree of node $u$ \\
		\hline
	\end{tabular}
	\vspace{-2mm}
\end{table}

Unless otherwise specified, we denote matrices with bold upper-case letters (i.e., $\mathbf{A}$), vectors with bold lower-case letters (i.e., $\mathbf{x}$) and scalars with italic lower-case letters (i.e., $c$). We use rules similar to NumPy in Python for matrix and vector indexing. $\mathbf{A}[i, j]$ represents the entry of matrix $\mathbf{A}$ at the $i$-th row and the $j$-th column. $\mathbf{A}[i, :]$ and $\mathbf{A}[:, j]$ represent the $i$-th row and the $j$-th column of matrix $\mathbf{A}$, respectively. We use superscript $^T$ to represent the transpose of a matrix, i.e., $\mathbf{A}^T$ is the transpose of matrix $\mathbf{A}$.

\subsection{Preliminaries}
\noindent \textbf{A -- Graph Convolutional Network.} We denote a graph as $\mathcal{G}=\{\mathcal{V}_{\mathcal{G}}, \mathbf{A}, \mathbf{X}\}$ where $\mathcal{V}_{\mathcal{G}}$ is the set of $n$ nodes in the graph (i.e., $n=|\mathcal{V}_{\mathcal{G}}|$), $\mathbf{A}$ is the $n \times n$ adjacency matrix and $\mathbf{X}\in\mathbb{R}^{n \times d_0}$ is the node feature matrix.

GCN is a typical graph neural network model that contains a stack of graph convolution layers. Based on the first-order Chebyshev polynomial, the graph convolution layer learns the latent node representations through the message-passing mechanism in two major steps. First, each node in the graph aggregates its own representation with the representations of its one-hop neighbors. Then, the aggregated representations are transformed through a fully-connected layer. Mathematically, for the $l$-th graph convolution layer, denoting its output node representations as $\mathbf{H}^{(l)}$ (we assume $\mathbf{H}^{(0)} = \mathbf{X}$ for notation consistency), it computes the latent representation with $\mathbf{H}^{(l)} = \sigma(\mathbf{\hat A} \mathbf{H}^{(l-1)} \mathbf{W}^{(l)}), \forall l\in\{1, \ldots, L\}$
% %\vspace{-1mm}
% \begin{equation}\label{eq:graph_convolution}
%     \mathbf{H}^{(l)} = \sigma(\mathbf{\hat A} \mathbf{H}^{(l-1)} \mathbf{W}^{(l)}), \forall l\in\{1, \ldots, L\}
%     %\vspace{-1mm}
% \end{equation}
where $\sigma(\cdot)$ is the nonlinear activation function, $\mathbf{W}^{(l)}\in \mathbb{R}^{d_{l-1} \times d_{l}}$ is the weight matrix, and $\mathbf{\hat A} = \mathbf{\tilde D}^{-\frac{1}{2}} (\mathbf{A} + \mathbf{I}) \mathbf{\tilde D}^{-\frac{1}{2}}$ is the renormalized graph Laplacian with $\mathbf{\tilde D}$ being the diagonal degree matrix of $\mathbf{A} + \mathbf{I}$.

\noindent \textbf{B -- The Rawlsian Difference Principle.} The Rawlsian difference principle is one of the major aspects of the equality principle in the theory of distributive justice by John Rawls~\cite{rawls1971theory}. The difference principle achieves equality by maximizing the welfare of the worst-off groups. When the Rawlsian difference principle is in its stability, the performance of all groups are balanced since there is no worst-off group whose welfare should be maximized and all groups preserve their status quo. For a machine learning model that predicts task-specific labels for each data sample, the welfare is often defined as the predictive accuracy of the model~\cite{rahmattalabi2019exploring, hashimoto2018fairness}. We denote (1) $\mathcal{D} = \{\mathcal{D}_1, \ldots, \mathcal{D}_h\}$ as a dataset that can be divided into $h$ different groups, (2) $J(\mathcal{D}, \mathbf{Y}, \theta)$ as the task-specific loss function that the model with parameters $\theta$ aims to minimize where $\mathbf{Y}$ is the model output and (3) $U(\cdot, \theta)$ as the utility function that measures the predictive accuracy over a set of samples using the model with parameters $\theta$. The Rawlsian difference principle can be mathematically formulated as
\vspace{-2mm}
\begin{equation}\label{eq:difference_principle}
\begin{aligned}
	& \textrm{min}_\theta \quad \textrm{Var}(\{U(\mathcal{D}_i, \theta)|i = 1, \ldots, h\}) \\
	& \textrm{s.t.} \qquad \theta = \argmin\ J(\mathcal{D}, \mathbf{Y}, \theta)
\end{aligned}
\vspace{-2mm}
\end{equation}
where $\textrm{Var}(\{U(\mathcal{D}_i, \theta)|i = 1, \ldots, h\})$ calculates the variance of the utilities of the groups $\{\mathcal{D}_1, \ldots, \mathcal{D}_h\}$.

\vspace{-3mm}
\subsection{Problem Definition}
% \noindent \textbf{C -- Problem Definition.} 
Despite superior performance of GCN in many tasks, GCN is often biased towards benefiting high-degree nodes. Following the overarching difference principle by John Rawls, we view the inconsistent predictive accuracy of GCN for high-degree and low-degree nodes %of different degrees 
as a distributive justice problem. However, directly enforcing the Rawlsian difference principle (Equation~\eqref{eq:difference_principle}) is nontrivial for two major reasons. First (C1), in many real-world applications, there could be multiple utility measures of interest. For example, in a classification task, an algorithm administrator might be interested in different measures like the classification accuracy, precision, recall and F1 score. Even if only one utility measure is considered, it is likely that the measure itself is non-differentiable, which conflicts with the end-to-end training paradigm of GCN. Second (C2), it is hard to decide whether a node is low-degree or high-degree by a clear threshold of degree value. A bad choice of the threshold value could even introduce more bias in calculating the average utilities. For example, if we set the threshold to be too large, the group of low-degree nodes might contain relatively high-degree nodes on which GCN achieves high utility. Then its average utility %of low-degree nodes 
will increase by including these relatively high-degree nodes. In this case, even when the GCN balances the utilities between the groups of low-degree nodes and high-degree nodes, many nodes with relatively low degrees still suffer from the issue of low predictive accuracy.

To address the first challenge (C1), we replace the utility function $U$ with the loss function $J$ as a proxy measure of predictive accuracy. The intuition lies in the design of the end-to-end training paradigm of GCN, in which we minimize the loss function in order to maximize the predictive accuracy of GCN. Thus, we aim to achieve a balanced loss in the stability of the Rawlsian difference principle. As for the second challenge (C2), instead of setting a hard threshold to split the groups of low-degree nodes and high-degree nodes, we split the node set $\mathcal{V}=\cup_{i=1}^{\textit{deg}_{\textrm{max}}} \mathcal{V}_{i}$ to a maximum of $\textit{deg}_{\textrm{max}}$ degree groups where $\mathcal{V}_i$ refers to the set of nodes whose degrees are equal to $i$. With that, we formally define the problem of enforcing the Rawlsian difference principle on GCN as follows.

% When the Rawlsian difference principle is enforced, the worst-off group (i.e., low-degree nodes in GCN) preserves its status quo and reaches the stability. In other words, the equality (i.e., fairness) in predictive accuracy is expected to be achieved among high-degree nodes and low-degree nodes. Formally, we define the problem of enforcing the Rawlsian difference principle on GCN as follows.

\vspace{-2mm}
\begin{prob}\label{prob:rawlsgcn}
	Enforcing the Rawlsian Difference Princple on GCN
\end{prob}
\vspace{-2mm}
\textbf{Input:} (1) an undirected graph $\mathcal{G}=\{\mathcal{V}_{\mathcal{G}}, \mathbf{A}, \mathbf{X}\}$; (2) an $L$-layer GCN with the set of weights $\theta$; (3) a task-specific loss function $J(\mathcal{G}, \mathbf{Y}, \theta)$ where $\mathbf{Y}$ is the model output.%; (4) a utility function $U(\mathcal{V}, \theta)$ \hh{different notation than V? since V is used to denote the entire node set. same modification for some places in section 3.1} \jian{how about we re-define the entire node set as $\mathcal{V}_{\mathcal{G}}$ so that $\mathcal{V}$ just denotes a set of nodes instead of entire node set?} that measures the predictive accuracy over a set of nodes $\mathcal{V}$ using the model with parameters $\theta$.

\textbf{Output:} a well-trained GCN that (1) minimizes the task specific loss $J(\mathcal{G}, \mathbf{Y}, \theta)$ given the input graph $\mathcal{G}$ and (2) achieves a balanced loss for all degree groups $\mathcal{V}_i$ ($i=1,\ldots,\textit{deg}_{\textrm{max}}$). %and (2) balances the utility between low-degree nodes $U(\mathcal{V}_{\textrm{low-degree}}, \theta)$ and high-degree nodes $U(\mathcal{V}_{\textrm{high-degree}}, \theta)$.

%% file: 03method.tex
\section{Methodology}\label{sec:method}
In this section, we propose a family of algorithms, namely \rawlsgcn, to enforce the Rawlsian difference principle on GCN. We first present analysis on the source of degree-related unfairness, which turns out to be rooted in the gradient of weight parameters in the GCN. Then we discuss how to compute doubly stochastic matrix which is the key to mitigate the degree-related unfairness. Based on that, we present a pre-processing method (\rawlsgcn-Graph) and an in-processing method (\rawlsgcn-Grad) to solve Problem~\ref{prob:rawlsgcn}.

\vspace{-3mm}
\subsection{Source of Unfairness}
% \hh{the logic of this paragraph is not very clear. maybe, we should stick with the narrative in intro, i.e., sth like ' since the key component in training GCN is the gradient matrix of W, we seek to understand the root cause of degree-related unfairness in GCN by analzying it.'} \jian{sure, i will make it stick with our narrative in intro and reorganize the math into theorem-proof style.} The intuition and rationality of analyzing the source of unfairness in the learning process is as follows. Existing works~\cite{wu2019net, tang2020investigating} improve the classification accuracy of low-degree nodes by introducing degree-specific weight parameters. This implies that a vanilla GCN (i.e., GCN proposed in \cite{kipf2017semi}) trained with gradient-based optimizer (e.g., stochastic gradient descent or Adam~\cite{kingma2015adam}) is biased towards high-degree nodes. Inspired by that, our goal is to \textit{understand whether the unfairness with respect to node degrees comes from the gradient of weight parameters in the vanilla GCN}.

The key to solve Problem~\ref{prob:rawlsgcn} is to understand why the loss of a GCN varies among nodes with different degrees after training. Since the key component in training a GCN is the gradient matrix of the weight parameters with respect to the loss function, we seek to understand the root cause of such degree-related unfairness by analyzing it mathematically. In this section, our detailed analysis (Theorem~\ref{thm:source_of_unfairness}) reveals the following fact: in a graph convolution layer, the gradient matrix $\frac{\partial J}{\partial \mathbf{W}}$\footnote{We use $J$ to represent $J(\mathcal{G}, \mathbf{Y}, \theta)$ for notation simplicity.} of the loss function $J$ with respect to the weight parameter $\mathbf{W}$ is equivalent to a weighted summation of the influence matrix of each node, weighted by its degree in input adjacency matrix $\mathbf{\hat A}$. 

\begin{thm}\label{thm:source_of_unfairness}
Suppose we have an input graph $\mathcal{G}=\{\mathcal{V}_{\mathcal{G}}, \mathbf{A}, \mathbf{X}\}$, the renormalized graph Laplacian $\mathbf{\hat A} = \mathbf{\tilde D}^{-\frac{1}{2}} (\mathbf{A} + \mathbf{I}) \mathbf{\tilde D}^{-\frac{1}{2}}$, a nonlinear activation function $\sigma()$ and an $L$-layer GCN that minimizes a task-specific loss function $J$. For any $l$-th hidden graph convolution ($\forall l\in\{1,\ldots,L\}$) layer, the gradient of the loss function $J$ with respect to the weight parameter $\mathbf{W}^{(l)}$ is a linear combination of the influence of each node weighted by its degree in the renormalized graph Laplacian.
\begin{equation}
    \frac{\partial J}{\partial \mathbf{W}^{(l)}} 
    = \sum_{j=1}^{n} 
        \textit{deg}_{\mathbf{\hat A}}(j) 
        \mathbf{I}_j^{\textrm{(row)}}
    = \sum_{i=1}^{n} 
        \textit{deg}_{\mathbf{\hat A}}(i) 
        \mathbf{I}_i^{\textrm{(col)}}
\end{equation}
where $\textit{deg}_{\mathbf{\hat A}}(i)$ is the degree of node $i$ in the renormalized graph Laplacian $\mathbf{\hat{A}}$, $\mathbf{I}_j^{\textrm{(row)}} = \big(\mathbf{H}^{(l-1)}[j, :]\big)^T \mathbb{E}_{i \sim p_{\mathcal{\hat N}(j)}}\bigg[\frac{\partial J}{\partial \mathbf{E}^{(l)}[i, :]}\bigg]$ is the row-wise influence matrix of node $j$, $\mathbf{I}_i^{\textrm{(col)}} = \bigg(\mathbb{E}_{j \sim p_{\mathcal{\hat N}(i)}}\big[\mathbf{H}^{(l-1)}[j, :]\big]\bigg)^T \frac{\partial J}{\partial \mathbf{E}^{(l)}[i, :]}$ is the column-wise influence matrix of node $i$, $\mathbf{H}^{(l-1)}$ is the input node embeddings of the hidden layer and $\mathbf{E}^{(l)} = \mathbf{\hat A} \mathbf{H}^{(l-1)} \mathbf{W}^{(l)}$ is the node embeddings before the nonlinear activation.
\end{thm}

\vspace{-3mm}
\begin{proof}
To compute the derivative of the objective function $J$ with respect to the weight $\mathbf{W}^{(l)}$ in the $l$-th graph convolution layer, by the graph convolution $\mathbf{H}^{(l)} = \sigma(\mathbf{\hat A} \mathbf{H}^{(l-1)} \mathbf{W}^{(l)})$ and the chain rule of matrix derivative, we have
\begin{equation}\label{eq:grad_chain_rule}
	\frac{\partial J}{\partial \mathbf{W}^{(l)}[i, j]} = \sum_{a=1}^n \sum_{b=1}^{d_l} \frac{\partial J}{\partial \mathbf{H}^{(l)} [a, b]} \frac{\partial \mathbf{H}^{(l)} [a, b]}{\partial \mathbf{W}^{(l)}[i, j]}
% \vspace{-2mm}
\end{equation}
where $d_l$ is the number of columns in $\mathbf{H}^{(l)}$. To compute Equation~\eqref{eq:grad_chain_rule}, a key term to compute is $\frac{\partial \mathbf{H}^{(l)} [a, b]}{\partial \mathbf{W}^{(l)}[i, j]}$. Denoting $\sigma '$ as the derivative of the activation function $\sigma$, by the graph convolution, we get 
\vspace{-1mm}
\begin{equation}\label{eq:grad_key_term}
\begin{aligned}
	\frac{\partial \mathbf{H}^{(l)} [a, b]}{\partial \mathbf{W}^{(l)}[i, j]} 
	& = \frac{\partial \sigma\big((\mathbf{\hat A} \mathbf{H}^{(l-1)} \mathbf{W}^{(l)})[a, b]\big)}{\partial (\mathbf{\hat A} \mathbf{H}^{(l-1)} \mathbf{W}^{(l)})[a, b]} 
		\frac{\partial (\mathbf{\hat A} \mathbf{H}^{(l-1)} \mathbf{W}^{(l)})[a, b]}{\partial \mathbf{W}^{(l)}[i, j]} \\
	& = \sigma'\big((\mathbf{\hat A} \mathbf{H}^{(l-1)} \mathbf{W}^{(l)})[a, b]\big)
		(\mathbf{\hat A} \mathbf{H}^{(l-1)})[a, i] \mathbbm{1}[b == j]
\end{aligned}
\vspace{-2mm}
\end{equation}
Combining Equations~\eqref{eq:grad_chain_rule} and \eqref{eq:grad_key_term}, we have
%\vspace{-2mm}
\begin{equation}\label{eq:grad_elementwise}
\begin{aligned}
	\frac{\partial J}{\partial \mathbf{W}^{(l)}[i, j]} 
% 	& = \sum_{a} \sum_{b} 
% 		\frac{\partial J}{\partial \mathbf{H}^{(l)} [a, b]} 
% 		\frac{\partial \mathbf{H}^{(l)} [a, b]}{\partial \mathbf{W}^{(l)}[i, j]} \\
% 	& = \sum_{a} 
% 		\frac{\partial J}{\partial \mathbf{H}^{(l)} [a, j]} 
% 		\sigma'\big((\mathbf{\hat A} \mathbf{H}^{(l-1)} \mathbf{W}^{(l)})[a, j]\big)
% 		(\mathbf{\hat A} \mathbf{H}^{(l-1)})[a, i] \\
% 	& = \sum_{a} 
% 		\big(\frac{\partial J}{\partial \mathbf{H}^{(l)}} \circ  
% 			\sigma'(\mathbf{\hat A} \mathbf{H}^{(l-1)} \mathbf{W}^{(l)})\big)[a, j]
% 		(\mathbf{\hat A} \mathbf{H}^{(l-1)})[a, i] \\
	& = \sum_{a} 
		\big( \mathbf{\hat A} \mathbf{H}^{(l-1)} \big)^T[i, a]
		\big(\frac{\partial J}{\partial \mathbf{H}^{(l)}} \circ  
			\sigma'(\mathbf{\hat A} \mathbf{H}^{(l-1)} \mathbf{W}^{(l)})\big)[a, j] \\
	& = \big(\mathbf{\hat A} \mathbf{H}^{(l-1)}\big)^T[i, :] 
		\big(\frac{\partial J}{\partial \mathbf{H}^{(l)}} \circ  
		\sigma'(\mathbf{\hat A} \mathbf{H}^{(l-1)} \mathbf{W}^{(l)})\big)[:, j]
\end{aligned}
% \vspace{-1mm}
\end{equation}
where $\circ$ represents the element-wise product. Writing Equation~\eqref{eq:grad_elementwise} into matrix form, we have
\vspace{-1mm}
\begin{equation}\label{eq:grad_matrix_raw}
	\frac{\partial J}{\partial \mathbf{W}^{(l)}} 
	= (\mathbf{H}^{(l-1)})^T 
	  \mathbf{\hat A}^T
	  \big(\frac{\partial J}{\partial \mathbf{H}^{(l)}} \circ  
		  \sigma'(\mathbf{\hat A} \mathbf{H}^{(l-1)} \mathbf{W}^{(l)})\big)
\vspace{-1mm}
\end{equation}
Let $\mathbf{E}^{(l)} = \mathbf{\hat A} \mathbf{H}^{(l-1)} \mathbf{W}^{(l)}$ denote the node embeddings before the nonlinear activation, i.e., $\mathbf{H}^{(l)} = \sigma(\mathbf{E}^{(l)})$. By the chain rule of matrix derivative, we have $\frac{\partial J}{\partial \mathbf{E}^{(l)}} = \frac{\partial J}{\partial \mathbf{H}^{(l)}} \circ \sigma'(\mathbf{\hat A} \mathbf{H}^{(l-1)} \mathbf{W}^{(l)})$. Then Equation~\eqref{eq:grad_matrix_raw} can be written as\footnote{A simplified result on linear GCN without nonlinear activation is shown in~\cite{guo2021orthogonal}, while our result (Equation~\eqref{eq:grad_matrix}) generalizes to GCN with \textit{arbitrary} differentiable nonlinear activation function.}
\vspace{-2mm}
\begin{equation}\label{eq:grad_matrix}
	\frac{\partial J}{\partial \mathbf{W}^{(l)}} 
	= (\mathbf{H}^{(l-1)})^T 
	  \mathbf{\hat A}^T
	  \frac{\partial J}{\partial \mathbf{E}^{(l)}}
\vspace{-2mm}
\end{equation}

To analyze the influence of each node on the gradient $\frac{\partial J}{\partial \mathbf{W}^{(l)}}$, 
%denoting $\mathbf{B} = \frac{\partial J}{\partial \mathbf{H}^{(l)}} \circ \sigma'(\mathbf{\hat A} \mathbf{H}^{(l-1)} \mathbf{W}^{(l)})$, 
we factorize Equation~\eqref{eq:grad_matrix} as follows.
\vspace{-2mm}
\begin{equation}\label{eq:grad_factorized}
	\frac{\partial J}{\partial \mathbf{W}^{(l)}} 
	= \sum_{i=1}^{n} \sum_{j=1}^{n} 
		\mathbf{\hat A}^T [i, j]
		(\mathbf{H}^{(l-1)}[i, :])^T
% 		\mathbf{B}
		\frac{\partial J}{\partial \mathbf{E}^{(l)}[j, :]}
\vspace{-2mm}
\end{equation}
Denoting the distribution $p_{\mathcal{\hat N}(i)}$ over the neighborhood of node $i$ in the renormalized graph Laplacian $\mathbf{\hat A}$ such that $p_{\mathcal{\hat N}(i)} (j) \propto \mathbf{\hat A}[i,j]=\mathbf{\hat A}[j,i], \forall j \sim p_{\mathcal{\hat N}(i)}$, we can rewrite Equation~\eqref{eq:grad_factorized} as 
\vspace{-2mm}
\begin{equation}\label{eq:grad_nodewise}
\begin{aligned}
	\frac{\partial J}{\partial \mathbf{W}^{(l)}} 
% 	& = \sum_{i=1}^{n} \sum_{j=1}^{n} 
% 		\mathbf{\hat A} [j, i] 
% 		(\mathbf{H}^{(l-1)}[i, :])^T 
% % 		\mathbf{B}
% 		\frac{\partial J}{\partial \mathbf{E}^{(l)}[j, :]} \\
	& = \sum_{j=1}^{n} 
		\textit{deg}_{\mathbf{\hat A}}(j) 
		\big(\mathbf{H}^{(l-1)}[j, :]\big)^T 
		\mathbb{E}_{i \sim p_{\mathcal{\hat N}(j)}}\bigg[
		  %  \mathbf{B}
		    \frac{\partial J}{\partial \mathbf{E}^{(l)}[i, :]}
		\bigg] \\
	& = \sum_{i=1}^{n} 
		\textit{deg}_{\mathbf{\hat A}}(i) 
		\bigg(\mathbb{E}_{j \sim p_{\mathcal{\hat N}(i)}}\big[\mathbf{H}^{(l-1)}[j, :]\big]\bigg)^T 
% 		\mathbf{B}
		\frac{\partial J}{\partial \mathbf{E}^{(l)}[i, :]}
\end{aligned}
\vspace{-2mm}
\end{equation}
where $\textit{deg}_{\mathbf{\hat A}}(i) = \sum_{i=1}^{n} \mathbf{\hat A}[i, j] = \sum_{i=1}^{n} \mathbf{\hat A}[j, i]$ is the degree of node $i$ in the renormalized graph Laplacian $\mathbf{\hat A}$. We define the row-wise influence matrix $\mathbf{I}_j^{\textrm{(row)}}$ of a node $j$ and the column-wise influence matrix $\mathbf{I}_i^{\textrm{(col)}}$ of a node $i$ as follows.
\vspace{-2mm}
\begin{equation}\label{eq:influence}
\begin{aligned}
    \mathbf{I}_j^{\textrm{(row)}} = 
	    \big(\mathbf{H}[j, :]\big)^T 
		\mathbb{E}_{i \sim p_{\mathcal{\hat N}(j)}}\bigg[
		  %  \mathbf{B}
		    \frac{\partial J}{\partial \mathbf{E}^{(l)}[i, :]}
		\bigg] \\
    \mathbf{I}_i^{\textrm{(col)}} = 
        \bigg(\mathbb{E}_{j \sim p_{\mathcal{\hat N}(i)}}\big[\mathbf{H}[j, :]\big]\bigg)^T 
% 		\mathbf{B}
		\frac{\partial J}{\partial \mathbf{E}^{(l)}[i, :]}
\end{aligned}
% \vspace{-2mm}
\end{equation}
Then Equation~\eqref{eq:grad_nodewise} can be written as the weighted summation over the (row-/column-wise) influence matrix of each node, weighted by its corresponding degree in the renormalized graph Laplacian $\mathbf{\hat A}$.
\vspace{-2mm}
\begin{equation}\label{eq:grad_influence}
    \frac{\partial J}{\partial \mathbf{W}^{(l)}} 
    = \sum_{j=1}^{n} 
        \textit{deg}_{\mathbf{\hat A}}(j) 
        \mathbf{I}_j^{\textrm{(row)}}
    = \sum_{i=1}^{n} 
        \textit{deg}_{\mathbf{\hat A}}(i) 
        \mathbf{I}_i^{\textrm{(col)}}
    \vspace{-2mm}
\end{equation}
\end{proof}

\begin{figure}
    \centering
    \includegraphics[width=.48\textwidth]{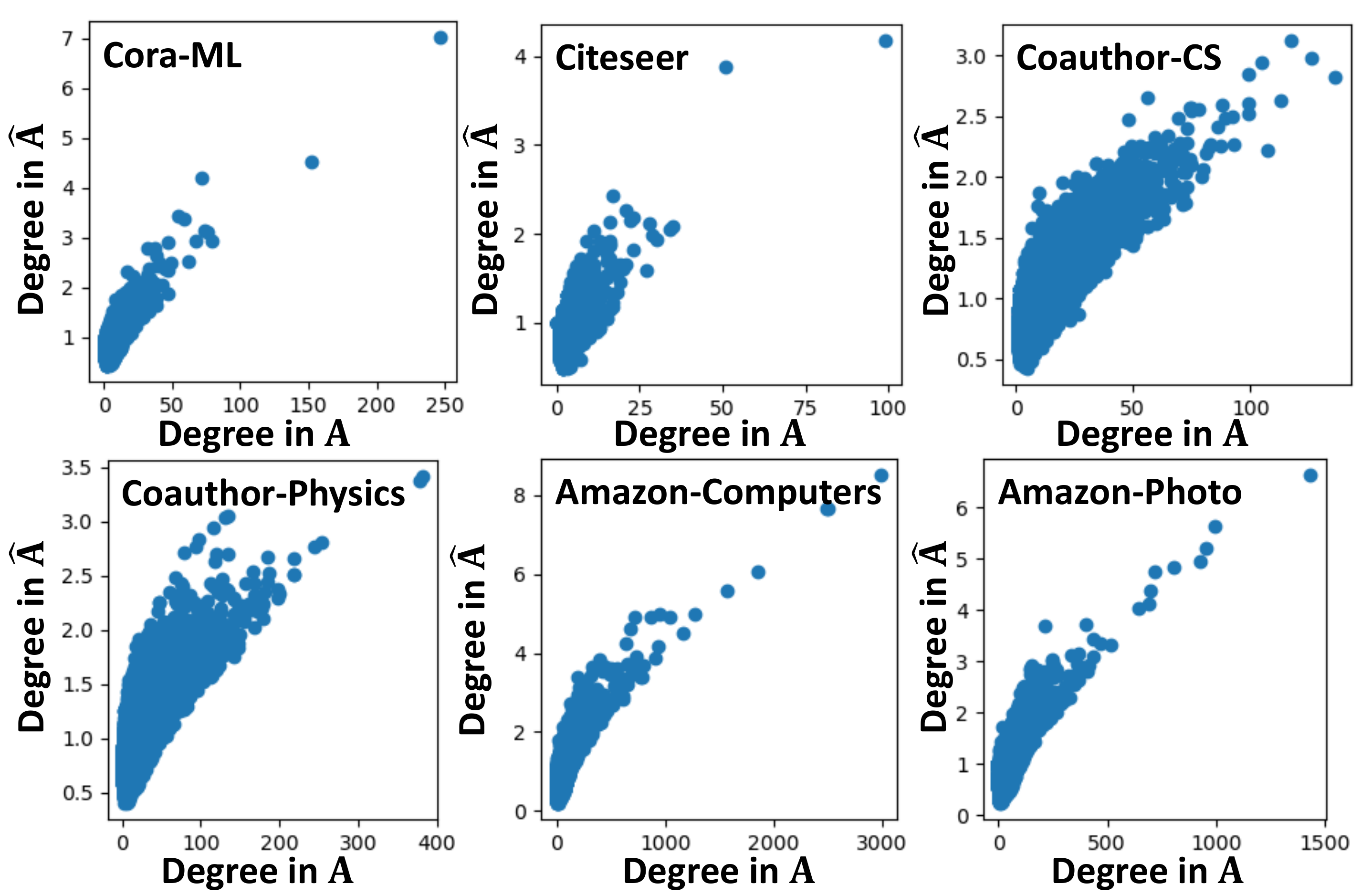}
    \vspace{-6mm}
    \caption{Node degrees in the original adjacency matrix $\mathbf{A}$ (x-axis) vs. the corresponding node degrees in the renormalized graph Laplacian $\mathbf{\hat A}$ (y-axis).}
    \label{fig:degree}
    \vspace{-6mm}
\end{figure}

% \hh{it would be better if we can organize the content until this point in the form of 'Theorem-Proof'. We first define I and B (and other necessary notations), then give out Theorem (Eq13), and finally a proof (either here or in Appendix)}

% \noindent \textbf{Summary:} By Equation~\eqref{eq:grad_influence}, 
\noindent \textbf{Remark:} By Theorem~\ref{thm:source_of_unfairness}, the gradient $\frac{\partial J}{\partial \mathbf{W}^{(l)}}$ is essentially equivalent to a linear combination of the influence matrix of each node in the graph, with the corresponding node degree $\textit{deg}_{\mathbf{\hat A}}$ in the renormalized graph Laplacian $\mathbf{\hat A}$ serving as the importance of the influence matrix. As a consequence, as long as the node degrees are not equal to a constant (e.g., $1$), the gradient $\frac{\partial J}{\partial \mathbf{W}^{(l)}}$ will favor the nodes with higher degrees in $\mathbf{\hat A}$. It is noteworthy that, even if $\mathbf{\hat A} = \mathbf{\tilde D}^{-\frac{1}{2}} (\mathbf{A} + \mathbf{I}) \mathbf{\tilde D}^{-\frac{1}{2}}$ is symmetrically normalized, such normalization only guarantees the largest eigenvalue of $\mathbf{\hat A}$ to be $1$%\hh{Jian, my brain might be a bit foggy. i think it should be 'upper bounded by 1' right? pls double check}
, whereas the degrees in $\mathbf{\hat A}$ are {\em not} constant as shown in Figure~\ref{fig:degree}. From the figure, we have two key observations: (1) the node degrees in $\mathbf{\hat A}$ are not equal among all nodes ($y$-axis); and (2) there is a positive correlation between a node degree in $\mathbf{A}$ ($x$-axis) and a node degree in $\mathbf{\hat A}$ ($y$-axis). Putting everything together, it means that the higher the node degree in $\mathbf{A}$, the more importance it has on the gradient of the weight matrix. This shows exactly why the vanilla GCN favors high-degree nodes while being biased against low-degree nodes.

\vspace{-3mm}
\subsection{Doubly Stochastic Matrix Computation}
In order to mitigate the node degree-related unfairness, the key idea is to normalize the importance of influence matrices (i.e., the node degrees in $\mathbf{\hat A}$) to $1$ in Equation~\eqref{eq:grad_influence}, such that each node will have equal importance in updating the weight parameters. To achieve that, Equation~\eqref{eq:grad_influence} naturally requires that the rows and columns of $\mathbf{\hat A}$ sums up to $1$, which is a doubly stochastic matrix.

Computing the doubly stochastic matrix is nontrivial. To achieve that, we adopt the Sinkhorn-Knopp algorithm, which is an iterative algorithm to balance a matrix into doubly stochastic form~\cite{sinkhorn1967concerning}. Mathematically speaking, given a non-negative $n \times n$ square matrix $\mathbf{A}$, it aims to find two $n\times n$ diagonal matrices $\mathbf{D}_1$ and $\mathbf{D}_2$ such that $\mathbf{D}_1\mathbf{A}\mathbf{D}_2$ is doubly stochastic. The intuition of Sinkhorn-Knopp algorithm is to learn a sequence of matrices whose rows and columns are alternatively normalized. Mathematically, defining $\mathbf{r}_0 = \mathbf{c}_0 = \mathbf{1}$ to be the column vectors of all $1$s and $\mathbf{x}^{-1}$ to be the operator for element-wise reciprocal, i.e., $\mathbf{x}^{-1}[i] = 1/\mathbf{x}[i]$, the Sinkhorn-Knopp algorithm alternatively calculates the following equations.
%\vspace{-2mm}
\begin{equation}\label{eq:sk_alg}
\begin{aligned}
    \mathbf{c}_{k+1} = (\mathbf{A}^T\mathbf{r}_k)^{-1} \quad \mathbf{r}_{k+1} = (\mathbf{A}\mathbf{c}_{k+1})^{-1}
\end{aligned}
%\vspace{-2mm}
\end{equation}
If the algorithm converges after $K$ iterations, the doubly stochastic form of $\mathbf{A}$ is $\mathbf{P} = \textrm{diag}(\mathbf{r}_K) \mathbf{A} \textrm{diag}(\mathbf{c}_K)$ where $\textrm{diag}(\mathbf{r}_K)$ and $\textrm{diag}(\mathbf{c}_K)$ diagonalize the column vectors $\mathbf{r}_K$ and $\mathbf{c}_K$ into diagonal matrices. The time and space complexities of computing the doubly stochastic matrix $\mathbf{P}$ is linear with respect to the size of input matrix $\mathbf{A}$.
%\vspace{-2mm}
\begin{lm}\label{lm:complexity}
(Time and space complexities) Let $\mathbf{A}$ be an $n \times n$ non-negative matrix with $m$ nonzero elements, the time complexity of the Sinkhorn-Knopp algorithm is $O(K(m+n))$ where $K$ is the number of iterations to convergence. It takes an additional $O(m+n)$ space.
\end{lm}
%\vspace{-4mm}
\begin{proof}
In the $k$-th iteration (where $1\leq k \leq K$) of the Sinkhorn-Knopp algorithm, it takes $O(m)$ time to compute $\mathbf{A}^T\mathbf{r}_k$ and $\mathbf{A}\mathbf{c}_{k+1}$. Then it takes $O(n)$ time for element-wise reciprocal. Thus, the time complexity of an iteration is $O(m+n)$. Since the algorithm takes $K$ iterations to converge, the overall time complexity is $O(K(m+n))$. Regarding the space complexity, it takes an additional $O(n)$ space to store vectors $\mathbf{c}_k$ and $\mathbf{r}_k$ in the $k$-th iteration and an additional $O(m)$ time to store the resulting doubly stochastic form of matrix $\mathbf{A}$. Thus, the overall space complexity is $O(m+n)$.
%\vspace{-2mm}
\end{proof}

Next, we prove the convergence of the Sinkhorn-Knopp algorithm in our setting. To this end, we first present the definition of the diagonal of a matrix corresponding to a column permutation.

%\vspace{-2mm}
\begin{defn}\label{defn:diagonal}
(Diagonal of a matrix corresponding to a column permutation~\cite{sinkhorn1967concerning}) Let $\mathbf{A}$ be an $n \times n$ square matrix and $\delta$ be a permutation over the set $\{1,\ldots, n\}$.
\begin{compactitem}
    \item [(1)] The sequence of elements $\mathbf{A}[1, \delta(1)], \mathbf{A}[2, \delta(2)], \ldots, \mathbf{A}[n, \delta(n)]$ is called the diagonal of $\mathbf{A}$ corresponding to $\delta$.
    \item [(2)] If $\delta$ is the identity, the diagonal is the main diagonal of $\mathbf{A}$.
    \item [(3)] If $\mathbf{A}[i, \delta(i)] > 0, \forall i$, the diagonal is a positive diagonal.
\end{compactitem}
\end{defn}

We then provide the formal definition of the support of a matrix in Definition~\ref{defn:support}.

%\vspace{-2mm}
\begin{defn}\label{defn:support}
(Support of a non-negative square matrix~\cite{sinkhorn1967concerning}) Let $\mathbf{A}$ be an $n \times n$ non-negative matrix, $\mathbf{A}$ is said to have support if $\mathbf{A}$ contains a positive diagonal. $\mathbf{A}$ is said to have total support if $\mathbf{A} \neq \mathbf{0}$ and if every positive element of $\mathbf{A}$ lies on a positive diagonal, where $\mathbf{0}$ is the zero matrix of the same size as $\mathbf{A}$. 
\end{defn}

By Definitions~\ref{defn:diagonal} and \ref{defn:support}, Sinkhorn and Knopp~\cite{sinkhorn1967concerning} proved the following result.

\vspace{-2mm}
\begin{thm}\label{thm:sk_convergence}
(Sinkhorn-Knopp theorem~\cite{sinkhorn1967concerning}) If $\mathbf{A}$ is an $n\times n$ non-negative matrix, the Sinkhorn-Knopp algorithm converges and finds the unique doubly stochastic matrix of the form $\mathbf{D}_1\mathbf{A}\mathbf{D}_2$ if and only if $\mathbf{A}$ has total support, where $\mathbf{D}_1$ and $\mathbf{D}_2$ are diagonal matrices.
\end{thm}
\vspace{-2mm}
\begin{proof}
Omitted.
\vspace{-2mm}
\end{proof}

Based on Theorem~\ref{thm:sk_convergence}, we give Lemma~\ref{lm:sk_gcn} which says that the Sinkhorn-Knopp algorithm always converges in our setting and finds the doubly stochastic matrix with respect to the renormalized graph Laplacian $\mathbf{\hat A}$.

%\vspace{-2mm}
\begin{lm}\label{lm:sk_gcn}
Given an adjacency matrix $\mathbf{A}$, if $\mathbf{\hat A} = \mathbf{\tilde D}^{-\frac{1}{2}}(\mathbf{A} + \mathbf{I}) \mathbf{\tilde D}^{-\frac{1}{2}}$ with $\mathbf{\tilde D}$ as the degree matrix of $\mathbf{A} + \mathbf{I}$, the Sinkhorn-Knopp algorithm always converges to find the unique doubly stochastic form of $\mathbf{\hat A}$.
\end{lm}
% \vspace{-2mm}
\begin{proof}
The key idea is to prove that $\mathbf{\hat A}$ has total support. Let $\textit{deg}_{\mathbf{\hat A}}(i)$ be the degree of node $i$ in $\mathbf{\hat A}$. It is trivial that $\mathbf{\hat A}$ has support because its main diagonal is positive. In order to prove that $\mathbf{\hat A}$ has total support, for any undirected edge $\mathbf{A}[i, j]$, we define a column permutation $\delta_{ij}$ as a permutation that satisfies (1) $\delta_{ij}(i)=j$; (2) $\delta_{ij}(j)=i$ and (3) $\delta_{ij}(k)=k, \forall k\neq i\ \textrm{and}\ k\neq j$. Then, for any edge $(i, j)$, by applying the permutation $\delta_{ij}$, the diagonal of $\mathbf{\hat A}$ corresponding to $\delta_{ij}$ is a positive diagonal due to the non-negativity of $\mathbf{\hat A}$. Thus, $\mathbf{\hat A}$ has total support because all positive elements of $\mathbf{\hat A}$ lie on positive diagonals, which completes the proof.
% \vspace{-2mm}
\end{proof}

Lemma~\ref{lm:sk_gcn} guarantees that we can always calculate the doubly stochastic form of $\mathbf{\hat A}$ using the Sinkhorn-Knopp algorithm, in order to ensure the equal importance of node influence in calculating the gradient of the weight parameter $\frac{\partial J}{\partial \mathbf{W}^{(l)}}$ in the $l$-th layer.

\vspace{-3mm}
\subsection{\rawlsgcn\ Algorithms}
If the gradient $\frac{\partial J}{\partial \mathbf{W}^{(l)}}$ is computed using the doubly stochastic matrix $\mathbf{\hat A}_{\textit{DS}}$ with respect to the renormalized graph Laplacian $\mathbf{\hat A}$, it is fair with respect to node degrees because all nodes will have equal importance in determining the gradient, i.e., their degrees in $\mathbf{\hat A}_{\textit{DS}}$ are all equal to $1$. Thus, a fair gradient with respect to node degrees can be calculated using Equation~\eqref{eq:grad_matrix} as follows.
\vspace{-2mm}
\begin{equation}\label{eq:fair_grad}
    \frac{\partial J}{\partial \mathbf{W}^{(l)}}_{\textrm{fair}}
    = (\mathbf{H}^{(l-1)})^T 
	  \mathbf{\hat A}_{\textit{DS}}^T
	  \frac{\partial J}{\partial \mathbf{E}^{(l)}}
%\vspace{-2mm}
\end{equation}
where $\mathbf{E}^{(l)} = \mathbf{\hat A} \mathbf{H}^{(l-1)} \mathbf{W}^{(l)}$.

Observing the computation of fair gradient in Equation~\eqref{eq:fair_grad}, it naturally leads to two methods to mitigate degree-related unfairness, including (1) a pre-processing method named \rawlsgcn-Graph which utilizes the doubly stochastic matrix $\mathbf{\hat A}_{\textit{DS}}$ as the input adjacency matrix, and (2) an in-processing method named \rawlsgcn-Grad that normalizes the gradient in GCN with Equation~\eqref{eq:fair_grad}.

\noindent \textbf{Method \#1: Pre-processing with \rawlsgcn-Graph.} If we are allowed to modify the input of the GCN whereas the model itself are fixed, we can precompute the input renormalized graph Laplacian $\mathbf{\hat A}$ into its doubly stochastic form $\mathbf{\hat A}_{\textit{DS}}$ and feed $\mathbf{\hat A}_{\textit{DS}}$ as the input of the GCN. With that, the gradient computed using Equation~\ref{eq:grad_matrix} is equivalent to Equation~\eqref{eq:fair_grad}. As a consequence, the Rawlsian difference principle is naturally ensured since all nodes in $\mathbf{\hat A}_{\textit{DS}}$ have the same degree. Given a graph $\mathcal{G} = \{\mathcal{V}, \mathbf{A}, \mathbf{X}\}$ and an $L$-layer GCN, \rawlsgcn-Graph adopts the following 3-step strategy.
\begin{compactitem}
    \item [\textbf{1.}] Precompute $\mathbf{\hat A} = \mathbf{\tilde D}^{-\frac{1}{2}} (\mathbf{A} + \mathbf{I}) \mathbf{\tilde D}^{-\frac{1}{2}}$ where $ \mathbf{\tilde D}$ is the diagonal degree matrix of $\mathbf{A}+\mathbf{I}$.
    \item [\textbf{2.}] Precompute $\mathbf{\hat A}_{\textit{DS}}$ by applying the Sinkhorn-Knopp algorithm on $\mathbf{\hat A}$.
    \item [\textbf{3.}] Input $\mathbf{\hat A}_{\textit{DS}}$ and $\mathbf{X}$ to the GCN for model training.
\end{compactitem}

\noindent\textbf{Method \#2: In-processing with \rawlsgcn-Grad.} If we have access to the model or the model parameters while the input data is fixed, we can precompute the doubly stochastic matrix $\mathbf{\hat A}_{\textit{DS}}$ and use $\mathbf{\hat A}_{\textit{DS}}$ to compute the fair gradient by Equation~\eqref{eq:fair_grad}. Then training GCN with the fair gradient ensures the Rawlsian difference principle because nodes of different degrees share the same importance in determining the gradient for gradient descent-based optimization. Given a graph $\mathcal{G} = \{\mathcal{V}, \mathbf{A}, \mathbf{X}\}$ and an $L$-layer GCN, the general workflow of \rawlsgcn-Grad is as follows. 
% \hh{i am a bit confused about the difference between these two algorithms -- if RAWlsGCN-Grad also uses A-ds, what is the difference from RawlsGCN-Graph?} \jian{RawlsGCN-Graph use A-ds as input while RawlsGCN-Grad use A as input but normalizing the gradient. I added a new step (step 3) in the following steps to make it more clear.}
\begin{compactitem}
    \item [\textbf{1.}] Precompute $\mathbf{\hat A} = \mathbf{\tilde D}^{-\frac{1}{2}} (\mathbf{A} + \mathbf{I}) \mathbf{\tilde D}^{-\frac{1}{2}}$ where $ \mathbf{\tilde D}$ is the diagonal degree matrix of $\mathbf{A}+\mathbf{I}$.
    \item [\textbf{2.}] Precompute $\mathbf{\hat A}_{\textit{DS}}$ by applying the Sinkhorn-Knopp algorithm on $\mathbf{\hat A}$.
    \item [\textbf{3.}] Input $\mathbf{\hat A}$ and $\mathbf{X}$ to the GCN for model training.
    \item [\textbf{4.}] For each graph convolution layer $l\in\{1,\ldots, L\}$, compute the fair gradient $\frac{\partial J}{\partial \mathbf{W}^{(l)}}_{\textrm{fair}}$ for each weight matrix $\mathbf{W}^{(l)}$ using Equation~\eqref{eq:fair_grad}.
    \item [\textbf{5.}] Update the model parameters $\mathbf{W}^{(l)}$ using the fair gradient $\frac{\partial J}{\partial \mathbf{W}^{(l)}}_{\textrm{fair}}$.
    \item [\textbf{6.}] Repeat steps 4-5 until the model converges. 
\end{compactitem}

An advantage of both \rawlsgcn-Graph and \rawlsgcn-Grad is that no additional time complexity will be incurred during the learning process of GCN, since we can precompute and store $\mathbf{\hat A}_{\textit{DS}}$. For \rawlsgcn-Graph, $\mathbf{\hat A}_{\textit{DS}}$ has the same number of nonzero elements as $\mathbf{\hat A}$ because computing $\mathbf{\hat A}_{\textit{DS}}$ is essentially rescaling each nonzero element in $\mathbf{\hat A}$. Thus, there is no additional time cost during the graph convolution operation with $\mathbf{\hat A}_{\textit{DS}}$. For \rawlsgcn-Grad, computing the fair gradient with Equation~\eqref{eq:fair_grad} enjoys the same time complexity as the vanilla gradient computation (Equation~\eqref{eq:grad_matrix}) because $\mathbf{\hat A}_{\textit{DS}}$ and $\mathbf{\hat A}$ has exactly the same number of nonzero entries. Thus, there is no additional time cost in big-O notation in optimizing the GCN parameters. In terms of the additional costs in computing $\mathbf{\hat A}_{\textit{DS}}$, it bears a linear time and space complexities with respect to the number of nodes and the number of edges in the graph as stated in Lemma~\ref{lm:complexity}.
% \hh{we should also mention the complexity of computing A-DS (which is O(m+n) in both time and space?} \jian{i will add a lemma in sec 3.3 (doubly stochastic matrix computation) and briefly mention the linear time complexity here}

%% file: 04experiment.tex
\section{Experimental Evaluation}\label{sec:experiment}
In this section, we evaluate our proposed \rawlsgcn\ methods in the task of semi-supervised node classification to answer the following questions:
\begin{compactitem}
\item [\textbf{Q1.}] How accurate are the \rawlsgcn\ methods in node classification?
\item [\textbf{Q2.}] How effective are the \rawlsgcn\ methods in debiasing?
\item [\textbf{Q3.}] How efficient are the \rawlsgcn\ methods in time and space?
\end{compactitem}

\vspace{-3mm}
\subsection{Experimental Settings}\label{subsec:settings}
\noindent \textbf{Datasets.} We utilize six publicly available real-world networks for evaluation. Their statistics, including the number of nodes, the number of edges, number of node features, number of classes and the median of node degrees (Median Deg.), are summarized in Table~\ref{tab:datasets}. For semi-supervised node classification, we use a fixed random seed to generate the training/validation/test sets for each network. The training set contains 20 nodes per class. The validation set and the test set contain 500 nodes and 1000 nodes, respectively.% \jiebo{why does the training set contain so much fewer nodes?} \jian{this is the typical setting for semi-supervised node classification with GCN.} \yan{since we care node degree distribution in this work, should we also talk about these dataset node degree distribution?} \jian{i can add a column for average degree. but it should be well recognized that these graphs' degree distribution are power law-like and long-tailed.}

\begin{table}[ht]
	\centering
    \vspace{-4mm}
	\caption{Statistics of datasets.}
	\vspace{-4mm}
	\label{tab:datasets}
	\resizebox{\linewidth}{!}{
	\begin{tabular}{c|ccccc}
		\hline
		\textbf{Name} & \textbf{Nodes} & \textbf{Edges} & \textbf{Features} & \textbf{Classes} & \textbf{Median Deg.} \\ %\textbf{Avg. Degree} \\ 
		\hline
		Cora-ML & 2,995 & 16,316 & 2,879 & 7 & 3 \\%5.45\\
		Citeseer & 3,327 & 9,104 & 3,703 & 6 & 2 \\ %2.74\\
		Coauthor-CS & 18,333 & 163,788 & 6,805 & 15 & 6 \\ %8.93 \\
		Coauthor-Physics & 34,493 & 495,924 & 8,415 & 5 & 10 \\ %14.38 \\
		Amazon-Computers & 13,752 & 491,722 & 767 & 10 & 22 \\ %35.76 \\
		Amazon-Photo & 7,650 & 238,162 & 745 & 8 & 22 \\ % 31.13 \\
		\hline
	\end{tabular}}
	\vspace{-4mm}
\end{table}

\noindent \textbf{Baseline Methods.} We compare the proposed \rawlsgcn\ methods with several baseline methods%.
, including GCN~\cite{kipf2017semi}, DEMO-Net~\cite{wu2019net}, DSGCN~\cite{tang2020investigating}, Tail-GNN~\cite{liu2021tail}, Adversarial Fair GCN (AdvFair)~\cite{bose2019compositional} and REDRESS~\cite{dong2021individual}. 
Detailed description of each baseline method is provided in Appendix\footnote{We use the official PyTorch implementation of GCN, Tail-GNN and REDRESS for experimental evaluation. For DEMO-Net, we implement our own PyTorch version and consult with the original authors for a sanity check. For DSGCN, we implement our own PyTorch version due to the lack of publicly available implementation.}.
% Detailed descriptions of each baseline method are as follows\footnote{We use the official implementation of vanilla GCN, Tail-GNN and REDRESS for experimental evaluation. For DEMO-Net, we implement our own PyTorch version and consult with the original authors for a sanity check. For DSGCN, we implement our own PyTorch version due to the lack of official implementation.}. 

\begin{table*}[ht]
	\centering
	\caption{Effectiveness for node classification. Lower is better for bias (in gray). Higher is better for accuracy (Acc., in percentage).}
	\vspace{-4mm}
	\label{tab:effectiveness}
	\begin{tabular}{ccacaca}
		\hline
		\multirow{2}{*}{\textbf{Method}} & \multicolumn{2}{c}{\textbf{Cora-ML}} & \multicolumn{2}{c}{\textbf{Citeseer}} &
		\multicolumn{2}{c}{\textbf{Coauthor-CS}} \\
		\cmidrule(){2-7}
		& \textbf{Acc.} & \textbf{Bias} & \textbf{Acc.} & \textbf{Bias} & \textbf{Acc.} & \textbf{Bias} \\
		\hline
		\textbf{GCN} & $80.10 \pm 0.812$ & $0.392 \pm 0.046$ 
		             & $68.60 \pm 0.341$ & $0.353 \pm 0.040$ 
		             & $93.28 \pm 0.194$ & $0.075 \pm 0.004$ \\
		\hline
		\textbf{DEMO-Net} & $61.60 \pm 0.687$ & $0.181 \pm 0.015$ 
		                  & $60.26 \pm 0.408$ & $0.315 \pm 0.022$ 
		                  & $65.90 \pm 0.583$ & $0.164 \pm 0.006$ \\
		\textbf{DSGCN} & $30.26 \pm 5.690$ & $8.003 \pm 2.766$ 
		               & $31.42 \pm 3.257$ & $6.887 \pm 1.947$ 
		               & $44.20 \pm 7.155$ & $1.460 \pm 0.397$ \\
		\textbf{Tail-GNN} & $78.54 \pm 0.582$ & $0.503 \pm 0.284$ 
		                  & $66.34 \pm 0.009$ & $0.655 \pm 0.382$ 
		                  & $92.66 \pm 0.196$ & $0.052 \pm 0.031$ \\
		\textbf{AdvFair} & $67.56 \pm 2.594$ & $10.01 \pm 2.480$ 
		                 & $50.26 \pm 6.277$ & $3.146 \pm 2.425$ 
		                 & $84.82 \pm 2.254$ & $12.26 \pm 6.797$ \\
		\textbf{REDRESS} & $75.70 \pm 0.620$ & $0.955 \pm 0.213$ 
		                 & $65.80 \pm 0.518$ & $0.944 \pm 0.077$ 
		                 & $92.44 \pm 0.233$ & $0.028 \pm 0.003$ \\
		\hline
		\textbf{\rawlsgcn-Graph (Ours)} & $76.96 \pm 1.098$ & $0.105 \pm 0.012$
		                            & $69.34 \pm 0.745$ & $0.196 \pm 0.013$ 
		                            & $92.52 \pm 0.264$ & $0.043 \pm 0.002$  \\
		\textbf{\rawlsgcn-Grad (Ours)} & $79.34 \pm 1.247$ & $0.232 \pm 0.065$
		                            & $68.81 \pm 0.462$ & $0.283 \pm 0.047$ 
		                            & $92.68 \pm 0.240$ & $0.058 \pm 0.007$  \\
		\hline
	\end{tabular}
    \\
 	\begin{tabular}{ccacaca}
		\hline
		\multirow{2}{*}{\textbf{Method}} & \multicolumn{2}{c}{\textbf{Coauthor-Physics}} & \multicolumn{2}{c}{\textbf{Amazon-Computers}} & \multicolumn{2}{c}{\textbf{Amazon-Photo}} \\
		\cmidrule(){2-7}
		& \textbf{Acc.} & \textbf{Bias} & \textbf{Acc.} & \textbf{Bias} & \textbf{Acc.} & \textbf{Bias} \\
		\hline
		\textbf{GCN} & $93.96 \pm 0.367$ & $0.023 \pm 0.001$ 
		             & $64.84 \pm 0.641$ & $0.353 \pm 0.026$ 
		             & $79.58 \pm 1.507$ & $0.646 \pm 0.038$ \\
		\hline
		\textbf{DEMO-Net} & $77.50 \pm 0.566$ & $0.084 \pm 0.010$ 
		                  & $26.48 \pm 3.455$ & $0.456 \pm 0.021$ 
		                  & $39.92 \pm 1.242$ & $0.243 \pm 0.013$ \\
		\textbf{DSGCN} & $79.08 \pm 1.533$ & $0.262 \pm 0.075$ 
		               & $27.68 \pm 1.663$ & $1.407 \pm 0.685$ 
		               & $26.76 \pm 3.387$ & $0.921 \pm 0.805$ \\
		\textbf{Tail-GNN} & OOM & OOM 
		                  & $76.24 \pm 1.491$ & $1.547 \pm 0.670$ 
		                  & $86.00 \pm 2.715$ & $0.471 \pm 0.264$ \\
		\textbf{AdvFair} & $87.44 \pm 1.132$ & $0.892 \pm 0.502$ 
		                 & $53.50 \pm 5.362$ & $4.395 \pm 1.102$ 
		                 & $75.80 \pm 3.563$ & $51.24 \pm 39.94$ \\
		\textbf{REDRESS} & $94.48 \pm 0.172$ & $0.019 \pm 0.001$ 
		                 & $80.36 \pm 0.206$ & $0.455 \pm 0.032$ 
		                 & $89.00 \pm 0.369$ & $0.186 \pm 0.030$ \\
		\hline
		\textbf{\rawlsgcn-Graph (Ours)} & $94.06 \pm 0.196$ & $0.016 \pm 0.000$ 
		                                & $80.16 \pm 0.859$ & $0.121 \pm 0.010$ 
		                                & $88.58 \pm 1.116$ & $0.071 \pm 0.006$ \\
		\textbf{\rawlsgcn-Grad (Ours)} & $94.18 \pm 0.306$ & $0.021 \pm 0.002$
		                               & $74.18 \pm 2.530$ & $0.195 \pm 0.029$ 
		                               & $83.70 \pm 0.672$ & $0.186 \pm 0.068$ \\
		\hline
	\end{tabular}
	\vspace{-3mm}
\end{table*}

%\vspace{-1mm}
\noindent \textbf{Metrics.} We use cross entropy as the loss function in semi-supervised node classification. To answer \textbf{Q1}, we evaluate the accuracy of node classification (i.e., Acc. in Table~\ref{tab:effectiveness}). For metrics in \textbf{Q2}, we define the bias w.r.t. the Rawlsian difference principle as the variance of degree-specific average cross entropy (i.e., AvgCE) to be consistent with the definition of Problem~\ref{prob:rawlsgcn}. Mathematically, it is defined as 
%\vspace{-1mm}
\begin{equation}
\begin{array}{c}
	\textrm{AvgCE}(k) = \mathbb{E}[\{\textrm{CE}(u), \forall\ \textrm{node}\ u\ \textrm{such that}\ \textit{deg}(u) = k\}] \\
	\textit{Bias} = \textrm{Var}(\{\textrm{AvgCE}(k), \forall\ \textrm{node degree}\ k\})
\end{array}
%\vspace{-1mm}
\end{equation}
where $\textrm{CE}(u)$ and $\textit{deg}(u)$ are the cross entropy and the degree of node $u$, respectively. To measure the efficiency (\textbf{Q3}), we count the number of learnable parameters (\# Param. in Table~\ref{tab:efficiency}), GPU memory usage in MB (Memory in Table~\ref{tab:efficiency}) and training time in seconds.

% \noindent \textbf{Metrics.} We use cross entropy as the loss function in semi-supervised node classification. To answer \textbf{Q1}, we evaluate the accuracy of node classification (i.e., Acc. in Table~\ref{tab:effectiveness}). For metrics in \textbf{Q2}, to be consistent with the definition of Problem~\ref{prob:rawlsgcn}, we define the bias w.r.t. the Rawlsian difference principle as $\textit{Bias} = \textrm{Var}(\{\textrm{AvgCE}(k), \forall\ \textrm{node degree}\ k\})$, which is the variance of degree-specific average cross entropy $\textrm{AvgCE}(k) = \mathbb{E}[\{\textrm{CE}(u), \forall\ \textrm{node}\ u\ \textrm{such that}\ \textit{deg}(u) = k\}]$ with $\textrm{CE}(u)$ and $\textit{deg}(u)$ being the cross entropy and the degree of node $u$, respectively. To measure the efficiency (\textbf{Q3}), we count the number of learnable parameters (\# Param. in Table~\ref{tab:efficiency}), GPU memory usage in MB (Memory in Table~\ref{tab:efficiency}) and training time in seconds.

\begin{figure*}
    \centering
    \begin{subfigure}[b]{.33\textwidth}
        \centering
         \includegraphics[width=\textwidth]{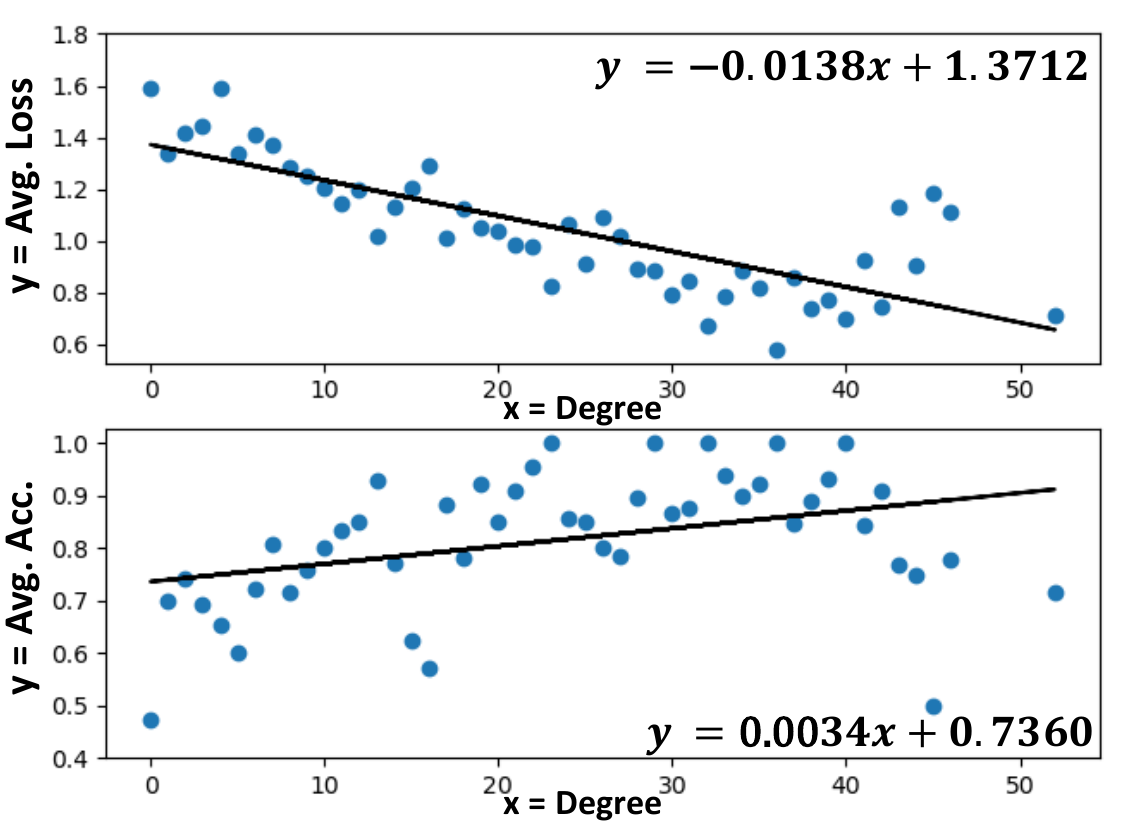}
         \caption{GCN}
         \label{fig:vis_exp_gcn}
    \end{subfigure}
    \begin{subfigure}[b]{.33\textwidth}
        \centering
         \includegraphics[width=\textwidth]{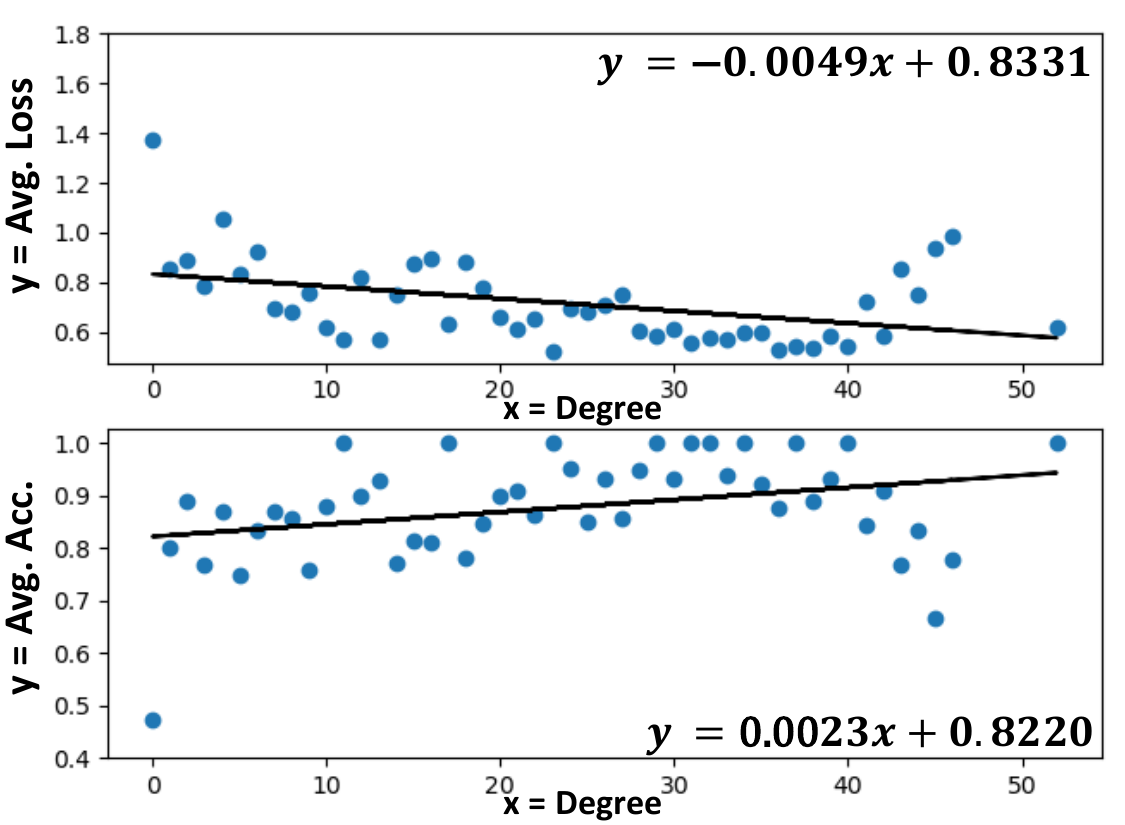}
         \caption{\rawlsgcn-Graph}
         \label{fig:vis_exp_rawlsgcn_graph}
    \end{subfigure}
    \begin{subfigure}[b]{.33\textwidth}
        \centering
         \includegraphics[width=\textwidth]{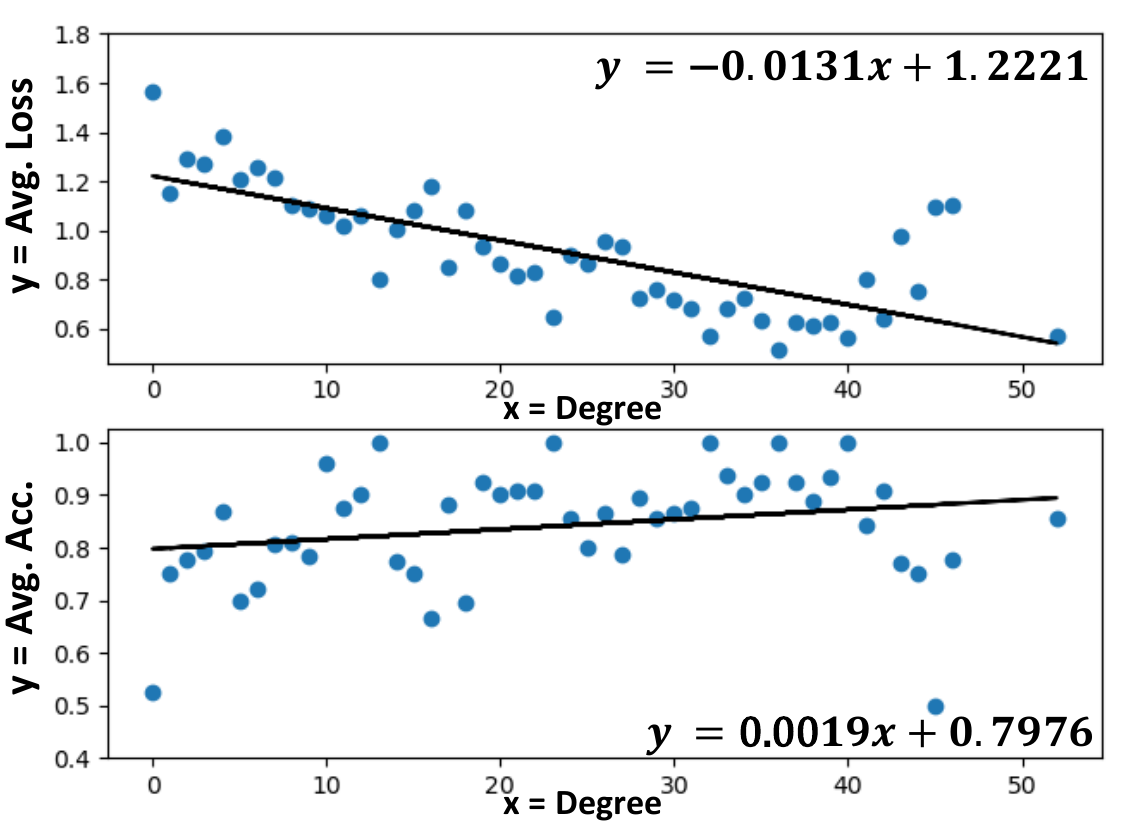}
         \caption{\rawlsgcn-Grad}
         \label{fig:vis_exp_rawlsgcn_grad}
    \end{subfigure}
    \vspace{-8mm}
    \caption{Visualization on how our proposed \rawlsgcn\ algorithms improve the performance of low-degree nodes on the Amazon-Photo dataset. (a) shows the results for vanilla GCN. (b) shows the results for \rawlsgcn-Graph. (c) shows the results for \rawlsgcn-Grad. Similar to Figure~\ref{fig:intro_example}, blue dots refer to the average loss (Avg. Loss) and average accuracy (Avg. Acc.) of a specific degree group in the top and bottom figures, respectively. Black lines are the regression lines of the blue dots in each figure. For a more clear visualization, we only consider the degree groups which contain more than five nodes.}
    \label{fig:vis_example}
    \vspace{-4mm}
\end{figure*}

\noindent \textbf{Parameter Settings.} For all methods, we use a $2$-layer GCN as the backbone model with the hidden dimension as $64$. We evaluate all methods on $5$ different runs and report their average performance. We train \rawlsgcn\ models for $100$ epochs without early stopping. For the purpose of the reproducibility, the random seeds for these $5$ runs are varied from $0$ to $4$. We use the Adam optimizer to train all methods. Unless otherwise specified, the default weight decay of the optimizer is set to $0.0005$. The detailed parameter settings (including learning rate, training epochs of baseline methods) are included in Appendix.

% For \rawlsgcn-Graph, \rawlsgcn-Grad and Adversarial Fair GCN, we search the learning rate that achieves the highest average classification accuracy in $[0.075, 0.05, 0.025, 0.01, 0.0075, 0.005, 0.0025]$. For DSGCN, due to its long running time, we search the learning rate in $[0.05, 0.025, 0.01, 0.005]$. For GCN, DEMO-Net, Tail-GNN and REDRESS, we use the suggested hyperparameters (including learning rate, weight decay, number of epochs and early stopping conditions) in the released source code. In addition, for REDRESS, we search the best choice of $\alpha$ (see details in \cite{dong2021individual}) in $[10^{-5}, 10^{-4}, 10^{-3}, 10^{-2}, 10^{-1}]$. For Adversarial Fair GCN, we train the model for $1000$ epochs with a patience of $200$ for early stopping. For DSGCN, due to long training time, we train the model for $200$ epochs without early stopping, which is consistent with the settings in GCN.

\noindent \textbf{Machine Configuration and Reproducibility.} All datasets are publicly available. All codes are programmed in Python 3.8.5 and PyTorch 1.9.0. All experiments are performed on a Linux server with 96 Intel Xeon Gold 6240R CPUs at 2.40 GHz and 4 Nvidia Tesla V100 SXM2 GPUs with 32 GB memory. We will release the source code of the proposed methods upon the publication of the paper.

\vspace{-2mm}
\subsection{Main Results}
% \subsection{Effectiveness Results}
\noindent \textbf{Effectiveness results.} 
The evaluation results are shown in Table~\ref{tab:effectiveness}. We do not report the results of Tail-GNN for the Coauthor-Physics dataset due to the out-of-memory (OOM) error. From the table, we can see that our proposed \rawlsgcn-Graph and \rawlsgcn-Grad are the only two methods that can consistently reduce the bias across all datasets. Though REDRESS reduces bias more than the proposed methods on the Coauthor-CS dataset, it bears a much higher bias than our proposed methods in other datasets. Surprisingly, on the  Amazon-Computers and Amazon-Photo datasets, our proposed \rawlsgcn-Graph and \rawlsgcn-Grad significantly improve the overall classification accuracy by mitigating the bias of low-degree nodes. This is because, compared with the vanilla GCN, the classification accuracy of low-degree nodes are increased while the classification accuracy of high-degree nodes are largely retained, resulting in the significantly improved overall accuracy.

In Figure~\ref{fig:vis_example}, we visualize how \rawlsgcn-Graph and \rawlsgcn-Grad benefit the low-degree nodes and balances the performance between low-degree nodes and high-degree nodes. From the figure, we observe that both \rawlsgcn-Graph (Figure~\ref{fig:vis_exp_rawlsgcn_graph}) and \rawlsgcn-Grad (Figure~\ref{fig:vis_exp_rawlsgcn_grad}) show lower average loss values and higher average classification accuracy compared with GCN (Figure~\ref{fig:vis_exp_gcn}). Moreover, to visualize how our proposed methods balance the performance between low-degree nodes and high-degree nodes, we perform linear regression on the average loss and average accuracy with respect to node degree. From the figure, we can see that the slope of the regression lines for \rawlsgcn-Graph and \rawlsgcn-Grad are flatter than the slope of the regression line in GCN.

% \subsection{Efficiency Results}
\noindent \textbf{Efficiency results.} 
We measure the memory and time consumption of all methods in Table~\ref{tab:efficiency}. In the table, GCN (100 epochs) and GCN (200 epochs) denote the GCN models trained with 100 epochs and 200 epochs, respectively. From the table, we have two key observations: (1) Compared with other baseline methods, \rawlsgcn-Graph and \rawlsgcn-Grad have fewer number of parameters to learn and are much more efficient in memory consumption; (2) \rawlsgcn-Graph and \rawlsgcn-Grad bear almost the same training time as the vanilla GCN with 100 epochs of training (i.e., GCN (100 epochs)) while all other baseline methods significantly increase the training time.

\begin{table}[t]
	\centering
	\caption{Efficiency of training a $2$-layer GCN on the Amazon-Photo dataset. Lower is better for all columns. GPU memory usage (Memory) is measured in MB. Training time is measured in seconds.}
	\vspace{-4mm}
	\label{tab:efficiency}
	\resizebox{\linewidth}{!}{\begin{tabular}{c|c|c|c}
		\hline 
		\textbf{Method} & \textbf{\# Param.} & \textbf{Memory} & \textbf{Training Time} \\
		\hline
		\textbf{GCN (100 epochs)} & $48,264$ & $1,461$ & $13.335$ \\
		\textbf{GCN (200 epochs)} & $48,264$ & $1,461$ & $28.727$ \\
		\hline
		\textbf{DEMO-Net} & $11,999,880$ & $1,661$ & $9158.5$ \\
		\textbf{DSGCN} & $181,096$ & $2,431$ & $2714.8$ \\
		\textbf{Tail-GNN} & $2,845,567$ & $2,081$ & $94.058$ \\
		\textbf{AdvFair} & $89,280$ & $1,519$ & $148.11$ \\
		\textbf{REDRESS} & $48,264$ & $1,481$ & $291.69$ \\
		\hline 
		\textbf{\rawlsgcn-Graph (Ours)} & $48,264$ & $1,461$ & $11.783$ \\
		\textbf{\rawlsgcn-Grad (Ours)} & $48,264$ & $1,461$ & $12.924$ \\
		\hline
	\end{tabular}}
	\vspace{-4mm}
\end{table}

\begin{table}[t]
	\centering
	\caption{Ablation study of different matrix normalization techniques on the Amazon-Photo. Lower is better for bias (in gray). Higher is better for accuracy (Acc., in percentage).}
	\vspace{-4mm}
	\label{tab:ablation}
	\resizebox{\linewidth}{!}{
	\begin{tabular}{c|c|c|a}
		\hline
		\textbf{Method} & \textbf{Normalization} & \textbf{Acc.} & \textbf{Bias} \\
		\hline\hline
		\multirow{4}{*}{\rawlsgcn-Graph} & \textbf{Row} & $87.98 \pm 0.791$ & $0.076 \pm 0.006$ \\
		                             & \textbf{Column} & $88.32 \pm 2.315$ & $0.138 \pm 0.112$ \\
		                             & \textbf{Symmetric}  & $89.12 \pm 0.945$ & $0.071 \pm 0.005$ \\
		\cline{2-4}
		                             & \textbf{Doubly Stochastic} & $88.58 \pm 1.116$ & $0.071 \pm 0.006$ \\
		
		\hline\hline
		\multirow{4}{*}{\rawlsgcn-Grad} & \textbf{Row} & $82.86 \pm 1.139$ & $0.852 \pm 0.557$ \\
		                             & \textbf{Column} & $84.96 \pm 1.235$ & $0.221 \pm 0.064$ \\
		                             & \textbf{Symmetric}  & $82.92 \pm 1.121$ & $0.744 \pm 0.153$ \\
		\cline{2-4}
		                             & \textbf{Doubly Stochastic} & $83.70 \pm 0.672$ & $0.186 \pm 0.068$ \\
		\hline
	\end{tabular}}
	\vspace{-6mm}
\end{table}

\vspace{-3mm}
\subsection{Ablation Study}
To evaluate the effectiveness of the doubly stochastic normalization on the renormalized graph Laplacian, we compare it with three other normalization methods, including row normalization, column normalization and symmetric normalization. As shown in Table~\ref{tab:ablation}, while all these normalization methods lead to similar accuracy (within 2\% difference), doubly stochastic normalization leads to a much smaller bias than others.
% doubly stochastic normalization can mitigate more bias than all other variants. \hh{consider to change sth more neural, e.g., while all these normalization methods lead to similar accuracy (within 2\% difference), doubly stochastic normalization leads to a much smaller bias.}Though doubly stochastic normalization does not achieve the highest accuracy, it is noteworthy that the gap is tiny compared with other normalization methods.

%% file: 05related.tex
\section{Related Work}\label{sec:related}
% In this section, we review related literature in (1) graph neural networks and (2) fair graph mining.

% \vspace{-3mm}
% \subsection{Graph Neural Network}
% \noindent 
\textbf{Graph neural network} 
% Graph neural network 
is an emerging research topic due to its strong empirical performance in many tasks like classification~\cite{jing2021hdmi}, regression~\cite{jing2021network} and recommendation~\cite{wang2019neural}. Bruna et al.~\cite{bruna2013spectral} propose the Graph Convolutional Neural Networks (GCNNs) by simulating the convolution operation in the spectrum of graph Laplacian. Kipf et al.~\cite{kipf2017semi} propose the Graph Convolutional Networks (GCNs) that aggregates the neighborhood information inspired by the localized first-order approximation of spectral graph convolution. Hamilton et al.~\cite{hamilton2017inductive} propose GraphSAGE that learns node representation in the inductive setting. Atwood et al.~\cite{atwood2016diffusion} propose a diffusion-based graph convolution operation that aggregates the neighborhood information through the graph diffusion process. Veli{\v{c}}kovi{\'c} et al.~\cite{velivckovic2018graph} introduce the multi-head self-attention mechanism into graph neural networks. Regarding graph neural networks for degree-aware representation learning, Wu et al.~\cite{wu2019net} propose two methods (i.e., hashing function and degree-specific weight function) to learn degree-specific representations for node and graph classification. Liu et al.~\cite{liu2021tail} learns robust embeddings for low-degree nodes (i.e., tail nodes) by introducing a novel neighborhood translation operation to predict missing information for the tail nodes and utilizing a discriminator to differentiate the head node embeddings and tail node embeddings. Different from \cite{wu2019net, liu2021tail}, our work directly improves the performance on low-degree nodes without relying on any additional degree-specific weights or degree-aware operations.

% \vspace{-3mm}
% \subsection{Fair Graph Mining}
% \noindent 
\textbf{Fair graph mining} 
% Fair graph mining 
has attracted much research attention recently. In terms of group fairness, it has been incorporated into several graph mining tasks, including ranking, clustering, node embedding and graph neural networks. Tsioutsiouliklis et al.~\cite{tsioutsiouliklis2021fairness} propose two fairness-aware PageRank algorithms (i.e., fairness-sensitive PageRank and locally fair PageRank) to ensure a certain proportion of total PageRank mass is assigned to nodes in a specific demographic group. Kleindessner et al.~\cite{kleindessner2019guarantees} propose fair spectral clustering which aims to ensure each demographic group is represented with the same fraction as in the whole dataset in all clusters. Bose et al.~\cite{bose2019compositional} study the compositional fairness for graph embedding, which aims to debias the node embeddings with respect to a combination of sensitive attributes. Rahman et al.~\cite{rahman2019fairwalk} modify the random walk procedure in node2vec~\cite{grover2016node2vec} so that neighbors in the minority demographic group enjoy a higher probability of being reached. Buyl et al.~\cite{buyl2020debayes} debias the embeddings by injecting bias information in the prior of conditional network embedding~\cite{kang2018conditional}. Dai et al.~\cite{dai2021say} promote statistical parity and equal opportunity in graph neural networks through adversarial learning. Individual fairness is another fundamental fairness notion. Kang et al.~\cite{kang2020inform} present the first effort on individually fair graph mining through Laplacian regularization on the pairwise node similarity matrix. Dong et al.~\cite{dong2021individual} incorporate ranking-based individual fairness into graph neural networks inspired by learning-to-rank. However, neither group fairness nor individual fairness is suitable to solve our problem. For group fairness, the low-degree nodes are often the majority group due to the long-tailed degree distribution. For individual fairness, it only considers fairness in node level by considering the similarity between two nodes, whereas our problem considers fairness in group level in which the groups are defined by node degrees. Many other fairness notions
%, such as counterfactual fairness, degree-related bias, Rawls' theory on difference principle\hh{it is a bit confusing to mention Rawls's principle here. is ref 22 the only paper that uses Rawls's principle in graph mining?}, 
are also studied in graph mining. Agarwal et al.~\cite{agarwal2021towards} exploit the connection between counterfactual fairness and stability to learn fair and robust node embeddings. Tang et al.~\cite{tang2020investigating} propose a RNN-based degree-specific weight generator with self-supervised learning to mitigate the degree-related bias. Our work differs from \cite{tang2020investigating} because we do not change the GCN architecture and do not introduce any degree-specific weights. Rahmattalabi et al.~\cite{rahmattalabi2019exploring} incorporate Rawlsian difference principle to the graph covering problem. Different from \cite{rahmattalabi2019exploring} that deals with a combinatorial problem, we are the first to introduce such fairness notion in graph neural networks without compromising its differentiable end-to-end paradigm. %\hh{i think we should say here or at some other places in the paper on why group fairness or individual fairness cannot solve our problem (e.g., for group fairness, low-degree nodes are actually often the majority group(s); for individual fairness, ...)}

%% file: 06conclusion.tex
\section{Conclusion}\label{sec:conclusion}
In this paper, we introduce the Rawlsian difference principle to Graph Convolutional Network (GCN), where we aim to mitigate degree-related unfairness. We formally define the problem of enforcing the Rawlsian difference principle on GCN as balancing the loss among groups of nodes with the same degree. Based on that, we reveal the mathematical root cause of degree-related unfairness by studying the computation of 
% how 
the gradient of weight parameters 
% is computed 
in GCN. Guided by its computation, we propose a pre-processing method named \rawlsgcn-Graph and an in-processing method named \rawlsgcn-Grad to mitigate the degree-related unfairness. Both methods rely on the Sinkhorn-Knopp algorithm in computing the doubly stochastic matrix of the graph Laplacian, which is guaranteed to converge in the context of GCN. Extensive evaluation on six real-world datasets
% We perform extensive evaluation on six real-world datasets to 
demonstrate the effectiveness and efficiency of our proposed methods. In the future, we will investigate how to generalize our methods to other graph neural networks.

%% file: 08acknowledgement.tex
J. Kang and H. Tong are partially supported by NSF (1947135 and 1939725), 
DARPA (HR001121C0165) 
and Army Research Office (W911NF2110088). 
The content of the information in this document does not necessarily reflect the position or the policy of the Government, and no official endorsement should be inferred.  The U.S. Government is authorized to reproduce and distribute reprints for Government purposes notwithstanding any copyright notation here on.

%% file: 07appendix.tex
\clearpage
\setcounter{secnumdepth}{0}
\section*{Appendix}\label{sec:appendix}

\subsection*{A -- \rawlsgcn-Graph vs. \rawlsgcn-Grad}
Here, we discuss the advantages and disadvantages of our proposed methods. We list the pros and cons of our proposed methods in Table~\ref{tab:comparison}. Moreover, though \rawlsgcn-Graph and \rawlsgcn-Grad have the same pre-processing procedure in the setting of fixed input graph, we believe that the general idea of normalizing the gradient in \rawlsgcn-Grad is useful for distributed training of extremely large graphs, in which a local subgraph of each node is often sampled using a (non-)deterministic sampler for feature aggregation and gradient computation. In this setting, the input graph is not deterministic during training and often asymmetric. Consequently, it is often impossible to precompute the doubly stochastic matrix for \rawlsgcn-Graph. However, we can still use the sampling distribution of local subgraph to calculate the normalized gradient using Eqs.~\eqref{eq:influence} and \eqref{eq:grad_influence}.
\begin{table}[ht]
    \centering
    \caption{Pros and cons of our proposed methods.}
    \label{tab:comparison}
    \vspace{-4mm}
    \begin{tabular}{|p{0.075\linewidth} | p{0.4\linewidth} | p{0.4\linewidth}|}
    \hline
    & \textbf{\rawlsgcn-Graph}  & \textbf{\rawlsgcn-Grad} \\ \hline
    \multirow{8}{*}{\textbf{Pros}}  
    % \textbf{Pros}
    & \begin{itemize}[
        align=left,
        itemindent=0pt,
        labelsep=0pt,
        labelwidth=1em,
        leftmargin=1em
    ]
        \item No need to modify the GNN model;
        \item Higher accuracy on graph with more diverse degree empirically;
        \item Smaller bias than \rawlsgcn-Grad.
    \end{itemize} 
    & \begin{itemize}[
        align=left,
        itemindent=0pt,
        labelsep=0pt,
        labelwidth=1em,
        leftmargin=1em
    ]
        \item Higher accuracy on graph with less diversity in node degree empirically;
        \item Able to work in use cases like distributed training of large graphs.
    \end{itemize} \\
    % & &  \\
    % & &  \\
    \hline
    \multirow{9}{*}{\textbf{Cons}}  
    & \begin{itemize}[
        align=left,
        itemindent=0pt,
        labelsep=0pt,
        labelwidth=1em,
        leftmargin=1em
    ]
        \item Lower accuracy on smaller graph/graph with less diversity in node degree empirically;
        \item May be unable to work in use cases like distributed training on extremely large graphs.
    \end{itemize}
    & \begin{itemize}[
        align=left,
        itemindent=0pt,
        labelsep=0pt,
        labelwidth=1em,
        leftmargin=1em
    ]
        \item Slightly higher bias than RawlsGCN-Graph;
        \item Need to change the optimizer of model.
    \end{itemize} \\
    % & &  \\
    \hline
    \end{tabular}
\end{table}

\subsection*{B -- Descriptions of Baseline Methods}
\begin{itemize}[
	align=left,
	leftmargin=1em,
	itemindent=0pt,
	labelsep=0pt,
	labelwidth=1em,
]
\item \textbf{GCN}~\cite{kipf2017semi} refers to the original Graph Convolutional Network (GCN) without fairness considerations. In our experiment, we adopt the same architecture as in \cite{kipf2017semi} but increasing the hidden dimension to 64 for a fair comparison.
\item \textbf{DEMO-Net}~\cite{wu2019net} uses multi-task graph convolution where each task learns degree-specific node representations in order to preserve the degree-specific graph structure. We use DEMO-Net with degree-specific weight function instead of hashing function due to its higher classification accuracy and better stability. For a fair comparison, we remove the components for order-free and seed-oriented representation learning as they are irrelevant to fairness w.r.t. node degree. 
\item \textbf{DSGCN}~\cite{tang2020investigating} mitigates degree-related bias by degree-specific graph convolution, which infers the degree-specific weights using a Recurrent Neural Network (RNN). We use 2 degree-specific graph convolution layers in DSGCN with the same hidden dimension settings as the vanilla GCN. We set the number of RNN cell to $10$ (i.e., nodes with degree larger than $10$ will share the same degree-specific weight), which is consistent with \cite{tang2020investigating}. We set the activation function of the RNN cell to tanh function. For a fair comparison, we only use the degree-specific graph convolution module (i.e., DSGCN in \cite{tang2020investigating}) for our experiments. 
\item \textbf{Tail-GNN}~\cite{liu2021tail} learns robust embedding for low-degree nodes (i.e., tail nodes) in an adversarial learning fashion with the novel neighborhood translation mechanism. It first generates forged tail nodes from nodes with degree higher than a certain threshold $k$. Then the neighborhood translation operation predicts the missing information of tail nodes and forged tail nodes by a translation model learned from head nodes. After that, a discriminator is applied to predict whether a node is head or tail based on the node representations. In our experiment, we set $k=5$ for forged tail nodes generation. If a training node $u$ has degree less than $k$, we do not generate the forged tail node using this training node.
\item \textbf{Adversarial Fair GCN} (AdvFair) is a variant of \cite{bose2019compositional} which ensures group fairness for graph embeddings in the compositional setting (i.e., for different combinations of sensitive attributes). We set the node degree as the sensitive attribute, i.e., nodes of the same degree form a demographic group. For a fair comparison, we compute the node embeddings using 2 graph convolution layers with ReLU activation, each of which has 64 hidden dimension. The `filtered' embeddings are computed by the filter, which is a 2-layer multi-layer perceptron (MLP) with 128 and 64 hidden dimensions, respectively. The discriminator is a 2-layer MLP where the first layer contains 64 hidden dimensions and the second layer predicts the sensitive attribute of each node. Both the filter and the discriminator use leaky ReLU as the activation function. A multi-class logistic regression is applied on the `filtered' embeddings for node classification.
\item \textbf{REDRESS}~\cite{dong2021individual} ensures individual fairness of graph neural network (GNN) by optimizing the similarity between the ranking lists of model input and output. In our experiment, we set the backbone GNN model as the vanilla GCN model described above.
\end{itemize}

\begin{table}[htbp]
	\centering
	\caption{Additional effectiveness results for node classification on Chameleon dataset. Lower is better for bias (the gray column). Higher is better for accuracy (Acc., in percentage).}
	\vspace{-4mm}
	\label{tab:additional_effectiveness}
	\begin{tabular}{cca}
		\hline
		\multirow{2}{*}{\textbf{Method}} & \multicolumn{2}{c}{\textbf{Chameleon}} \\
		\cmidrule(){2-3}
		& \textbf{Acc.} & \textbf{Bias} \\
		\hline
		\textbf{GCN} & $60.09 \pm 1.047$ & $0.504 \pm 0.118$ \\
		\hline
		\textbf{DEMO-Net} & $63.77 \pm 0.955$ & $0.352 \pm 0.015$ \\
		\textbf{DSGCN} & $47.76 \pm 1.978$ & $0.129 \pm 0.019$ \\
		\textbf{Tail-GNN} & $58.73 \pm 1.794$ & $1.040 \pm 0.654$ \\
		\textbf{AdvFair} & $42.54 \pm 8.499$ & $0.276 \pm 0.194$ \\
		\textbf{REDRESS} & $23.77 \pm 2.745$ & $0.020 \pm 0.005$ \\
		\hline
		\textbf{\rawlsgcn-Graph (Ours)} & $50.61 \pm 0.526$ & $0.098 \pm 0.007$ \\
		\textbf{\rawlsgcn-Grad (Ours)} & $45.92 \pm 2.741$ & $0.138 \pm 0.062$ \\
		\hline
	\end{tabular}
	\vspace{-3mm}
\end{table}

\begin{table*}[ht]
	\centering
	\caption{Additional ablation study of different matrix normalization techniques. Lower is better for bias (i.e., the gray column). Higher is better for accuracy (Acc.).}
	\vspace{-4mm}
	\label{tab:additional_ablation}
	\begin{tabular}{cccacaca}
		\hline
		\multirow{2}{*}{\textbf{Method}} & \multirow{2}{*}{\textbf{Normalization}} & \multicolumn{2}{c}{\textbf{Cora-ML}} & \multicolumn{2}{c}{\textbf{Citeseer}} & \multicolumn{2}{c}{\textbf{Coauthor-CS}} \\
		\cmidrule(){3-8}
		& & \textbf{Acc.} & \textbf{Bias} & \textbf{Acc.} & \textbf{Bias} & \textbf{Acc.} & \textbf{Bias} \\
		\hline\hline
		\multirow{4}{*}{\rawlsgcn-Graph} 
		& \textbf{Row} 
		& $79.74 \pm 0.320$ & $0.098 \pm 0.004$ 
		& $69.18 \pm 0.595$ & $0.240 \pm 0.013$ 
		& $92.78 \pm 0.331$ & $0.052 \pm 0.002$ \\
		& \textbf{Column} 
		& $76.78 \pm 1.360$ & $0.260 \pm 0.330$  
		& $69.12 \pm 0.781$ & $0.243 \pm 0.093$ 
		& $92.44 \pm 0.609$ & $0.049 \pm 0.012$ \\
		& \textbf{Symmetric} 
		& $77.04 \pm 1.606$ & $0.109 \pm 0.015$ 
		& $69.20 \pm 0.735$ & $0.196 \pm 0.014$ 
		& $92.56 \pm 0.120$ & $0.042 \pm 0.001$ \\
		\cline{2-8}
		& \textbf{Doubly Stochastic} 
		& $76.98 \pm 1.098$ & $0.105 \pm 0.012$ 
		& $69.34 \pm 0.745$ & $0.196 \pm 0.013$ 
		& $92.52 \pm 0.264$ & $0.043 \pm 0.002$ \\
		\hline\hline
		\multirow{4}{*}{\rawlsgcn-Grad} 
		& \textbf{Row} 
		& $79.78 \pm 0.349$ & $0.230 \pm 0.017$ 
		& $68.64 \pm 0.215$ & $0.274 \pm 0.036$ 
		& $92.92 \pm 0.440$ & $0.069 \pm 0.006$ \\
		& \textbf{Column} 
		& $79.94 \pm 0.599$ & $0.253 \pm 0.077$ 
		& $68.48 \pm 0.204$ & $0.302 \pm 0.049$ 
		& $92.78 \pm 0.407$ & $0.058 \pm 0.006$ \\
		& \textbf{Symmetric} 
		& $79.68 \pm 0.458$ & $0.199 \pm 0.008$ 
		& $68.68 \pm 0.248$ & $0.286 \pm 0.042$ 
		& $93.00 \pm 0.341$ & $0.063 \pm 0.006$ \\
		\cline{2-8}
		& \textbf{Doubly Stochastic} 
		& $79.34 \pm 1.247$ & $0.232 \pm 0.065$ 
		& $68.81 \pm 0.462$ & $0.283 \pm 0.047$ 
		& $92.68 \pm 0.240$ & $0.058 \pm 0.007$ \\
		\hline
	\end{tabular}
	
	\begin{tabular}{cccaca}
		\hline
		\multirow{2}{*}{\textbf{Method}} & \multirow{2}{*}{\textbf{Normalization}} & \multicolumn{2}{c}{\textbf{Coauthor-Physics}} & \multicolumn{2}{c}{\textbf{Amazon-Computers}} \\
		\cmidrule(){3-6}
		& & \textbf{Acc.} & \textbf{Bias} & \textbf{Acc.} & \textbf{Bias} \\
		\hline\hline
		\multirow{4}{*}{\rawlsgcn-Graph} 
		& \textbf{Row} 
		& $94.36 \pm 0.488$ & $0.013 \pm 0.000$ 
		& $78.54 \pm 1.125$ & $0.092 \pm 0.013$ \\
		& \textbf{Column} 
		& $93.98 \pm 0.508$ & $0.016 \pm 0.003$ 
		& $78.18 \pm 4.354$ & $0.196 \pm 0.106$ \\
		& \textbf{Symmetric} 
		& $93.98 \pm 0.248$ & $0.016 \pm 0.000$ 
		& $80.22 \pm 0.803$ & $0.126 \pm 0.012$ \\
		\cline{2-6}
		& \textbf{Doubly Stochastic} 
		& $94.06 \pm 0.196$ & $0.016 \pm 0.000$ 
		& $80.16 \pm 0.859$ & $0.121 \pm 0.010$ \\
		\hline\hline
		\multirow{4}{*}{\rawlsgcn-Grad} 
		& \textbf{Row} 
		& $94.08 \pm 0.204$ & $0.027 \pm 0.001$ 
		& $63.46 \pm 1.376$ & $0.453 \pm 0.039$ \\
		& \textbf{Column} 
		& $94.26 \pm 0.294$ & $0.020 \pm 0.002$ 
		& $75.48 \pm 1.273$ & $0.218 \pm 0.033$ \\
		& \textbf{Symmetric} 
		& $94.30 \pm 0.346$ & $0.021 \pm 0.001$ 
		& $66.42 \pm 0.584$ & $0.353 \pm 0.021$ \\
		\cline{2-6}
		& \textbf{Doubly Stochastic} 
		& $94.18 \pm 0.306$ & $0.021 \pm 0.002$ 
		& $74.18 \pm 2.530$ & $0.195 \pm 0.029$ \\
		\hline
	\end{tabular}
	%}
	\vspace{-4mm}
\end{table*}

\subsection*{C -- Parameter Settings} 
In this section, we provide additional parameter settings for the purpose of reproducibility, including the settings for learning rate and training epochs of baseline methods. For \rawlsgcn-Graph, \rawlsgcn-Grad and Adversarial Fair GCN, we search the learning rate that achieves the highest average classification accuracy in the set of $\{0.075, 0.05, 0.025, 0.01, 0.0075, 0.005, 0.0025\}$. For DSGCN, due to its long running time, we search the learning rate in the set of $\{0.05, 0.025, 0.01, 0.005\}$. For GCN, DEMO-Net, Tail-GNN and REDRESS, we use the suggested hyperparameters (including learning rate, weight decay, number of epochs and early stopping conditions) in the released source code. In addition, for REDRESS, we search the best choice of $\alpha$ (see details in \cite{dong2021individual}) in the range of $\{10^{-5}, 10^{-4}, 10^{-3}, 10^{-2}, 10^{-1}\}$. For Adversarial Fair GCN, we train the model for $1000$ epochs with a patience of $200$ for early stopping. For DSGCN, due to long training time, we train the model for $200$ epochs without early stopping, which is consistent with the settings in GCN.

\subsection*{D -- Additional Results on Heterophilic Graph}
Different from datasets listed in Section~\ref{subsec:settings}, a heterophilic graph consists of linked nodes that are likely to have dissimilar features or different class labels. We conduct additional experiments on a commonly used heterophilic graph named Chameleon dataset~\cite{pei2019geom}. It contains $2,277$ nodes which are Wikipedia pages about chameleons. Each node has $2,325$ features which correspond to informative nouns in the Wikipedia pages. Two nodes are connected if they have mutual link(s) between two pages, which forms $36,101$ edges. We follow the same experimental settings in Section~\ref{subsec:settings} for (1) generating training/validation/test sets, (2) training model and (3) evaluating results. The experimental result on Chameleon dataset is listed in Table~\ref{tab:additional_effectiveness}. From the table, we have the following observations. (1) Though REDRESS achieves smaller bias than \rawlsgcn-Graph and \rawlsgcn-Grad, its accuracy is severely reduced compared to GCN; (2) Though DSGCN outperforms \rawlsgcn-Grad with slightly higher accuracy and smaller bias, it fails to outperform \rawlsgcn-Graph; (3) Except for these two cases, our proposed methods achieve the smallest bias compared with all other baseline methods. Overall, our methods still achieve the best trade-off between the accuracy and bias.

\subsection*{E -- Additional Ablation Study Results}
We provide additional ablation study on other datasets listed in Section~\ref{subsec:settings}. From the Table~\ref{tab:additional_ablation}, we observe that, although row normalization and symmetric normalization outperforms the doubly stochastic normalization in some cases, it can also increase the bias in other cases, e.g., Amazon-Computers and Amazon-Photo (shown in Table~\ref{tab:ablation} in Section~\ref{subsec:settings}). All in all, the doubly stochastic normalization is the best one that (1) consistently mitigates bias for both RawlsGCN-Graph and RawlGCN-Grad, and (2) achieves good balance between accuracy and fairness.